\documentclass[11pt]{article}
\usepackage{geometry}
\usepackage{lipsum,booktabs}
\usepackage{amsmath,mathrsfs,amssymb,amsfonts,bm,bbm,enumitem}
\usepackage{rotating,verbatim,subfig}
\usepackage{pdflscape}
\usepackage{tcolorbox,xcolor}
\usepackage{hyperref,url,dsfont,nicefrac,natbib}
\allowdisplaybreaks
\usepackage{appendix}
\usepackage{multirow,makecell,tabularx}

\usepackage{algorithmic,algorithm}

\usepackage{standalone}     
\usepackage{preview}

\renewcommand{\tilde}{\widetilde}
\renewcommand{\hat}{\widehat}

\def \E {\mathbb{E}}

\def \H {\mathcal{H}}

\def \O {\mathcal{O}}

\def \R {\mathbb{R}}
\def \S {\mathcal{S}}
\def \T {\top}

\def \X {\mathcal{X}}

\def \v {\mathbf{v}}

\def \x {\mathbf{x}}
\def \y {\mathbf{y}}
\def \z {\mathbf{z}}

\def \Ecal {\mathcal{E}}

\def \Ot {\tilde{\O}}

\def \thetah {\hat{\theta}}

\def \base {\mathtt{base}\mbox{-}\mathtt{regret}}
\def \meta {\mathtt{meta}\mbox{-}\mathtt{regret}}
\def \epsilon {\varepsilon}

\usepackage{mathtools}
\let\norm\undefined 
\DeclarePairedDelimiter\norm{\lVert}{\rVert}
\DeclarePairedDelimiter\abs{\lvert}{\rvert}
\newcommand\inner[2]{\langle #1, #2 \rangle}

\DeclareMathOperator*{\argmax}{arg\,max}

\usepackage{amsthm}

\newtheorem{myThm}{Theorem}

\newtheorem{myLemma}{Lemma}

\theoremstyle{definition}

\newtheorem{myRemark}{Remark}

\usepackage{graphicx,color} 

\definecolor{wine_red}{RGB}{228,48,64}
\definecolor{DSgray}{cmyk}{0,1,0,0}

\DeclarePairedDelimiter\ceil{\lceil}{\rceil}
\DeclarePairedDelimiter\floor{\lfloor}{\rfloor}

\def \DReg {\textnormal{D-Regret}}

\def \meta {\mathtt{meta}\text{-}\mathtt{regret}}

\def \epsilon {\varepsilon}

\newcommand \sbr[1]{\left( #1 \right)}

\let\paragraph\relax

\newcommand \paragraph[1]{\vspace{2mm}\noindent \textbf{#1}~~}

\hypersetup{
    colorlinks,
    breaklinks,
    urlcolor = black,
    linkcolor = blue,
    citecolor = blue,
}
\usepackage{authblk}

\begin{document}

\title{A Simple Approach for Non-stationary Linear Bandits}

\author{Peng Zhao, Lijun Zhang, Yuan Jiang, Zhi-Hua Zhou}
\affil{National Key Laboratory for Novel Software Technology\\
Nanjing University, Nanjing 210023, China}
\date{}
\maketitle

\begin{abstract}
This paper investigates the problem of non-stationary linear bandits, where the unknown regression parameter is evolving over time. Existing studies develop various algorithms and show that they enjoy an $\Ot(T^{2/3}P_T^{1/3})$ dynamic regret, where $T$ is the time horizon and $P_T$ is the path-length that measures the fluctuation of the evolving unknown parameter. In this paper, we discover that a serious technical flaw makes their results ungrounded, and then present a fix, which gives an $\Ot(T^{3/4}P_T^{1/4})$ dynamic regret without modifying original algorithms. Furthermore, we demonstrate that instead of using sophisticated mechanisms, such as sliding window or weighted penalty, a simple restarted strategy is sufficient to attain the same regret guarantee. Specifically, we design an UCB-type algorithm to balance exploitation and exploration, and restart it periodically to handle the drift of unknown parameters. Our approach enjoys an $\Ot(T^{3/4}P_T^{1/4})$ dynamic regret. Note that to achieve this bound, the algorithm requires an oracle knowledge of the path-length $P_T$. Combining the bandits-over-bandits mechanism by treating our algorithm as the base learner, we can further achieve the same regret bound in a parameter-free way. Empirical studies also validate the effectiveness of our approach.
\end{abstract}

\section{Introduction}
\label{sec:SLB-intro}
Multi-Armed Bandits (MAB)~\citep{Robbins:52} models the sequential decision-making with partial information, where the player requires to choose one of the $K$ slot machines at each iteration in order to maximize the cumulative reward. MAB is a paradigmatic instance of the exploration versus exploitation trade-offs, which is fundamental in many areas of artificial intelligence, such as reinforcement learning~\citep{book:reinforcement-learning} and evolutionary algorithms~\citep{survey'13:EE-evolution}.  

In many real-world decision-making problems, each arm is usually associated with certain side information. Therefore, researchers start to formulate structured bandits in which the reward distributions of each arm are connected by a common but unknown parameter. Particularly, Stochastic Linear Bandits (SLB) has received much attention~\citep{JMLR'02:Auer-linear-bandits,NIPS'07:bandit-lower-bound,AISTATS'11:Chu-linear-bandits,NIPS'11:AY-linear-bandits,COLT'19:Yuan-linear-bandits}. In SLB, at iteration $t$, the player makes a decision $X_t$ from a feasible set $\X \subseteq \R^d$, and then observes the reward $r_t$ satisfying 
\begin{equation}
  \label{eq:LB-model-assume}
  \E[r_t|X_t] = X_t^\T \theta_*,
\end{equation}
where $\theta_*$ is an unknown regression parameter. The goal of the player is to minimize the (pseudo) regret,
\begin{equation}
  \label{eq:regret-stationary}
  \mbox{Regret}_T = T\max_{\x \in \X} \x^{\T}\theta_* - \sum_{t=1}^{T} X_t^\T \theta_*.
\end{equation}
The stochastic linear bandits problem is well-studied in literatures. By exploiting the tool of upper confidence bounds, various approaches demonstrate an $\Ot(d\sqrt{T})$ regret~\citep{NIPS'07:bandit-lower-bound,NIPS'11:AY-linear-bandits},\footnote{We adopt the notation of $\Ot(\cdot)$ to suppress logarithmic factors in the time horizon $T$.} which matches the $\Omega(d\sqrt{T})$ lower bound established by~\citet{NIPS'07:bandit-lower-bound}, up to $\log T$ factors. 

However, the observation model~\eqref{eq:LB-model-assume} assumes that the regression parameter $\theta_*$ is fixed, which is unfortunately hard to satisfy in real-life applications, because data are usually collected in non-stationary environments. For instance, in recommender systems the regression parameter models customers' interests, which could vary over time when customers look through product pages. Therefore, it is crucial to facilitate stochastic linear bandits with capability of handling non-stationarity.

To address above issue,~\citet{AISTATS'19:window-LB} proposed the \emph{non-stationary} linear bandits model, which assumes the reward $r_t$ satisfies
\begin{equation*}
  \label{eq:LB-model-assume-dynamic}
  \E[r_t|X_t] = X_t^\T \theta_t,
\end{equation*}
where $\theta_t$ is the unknown regression parameter at iteration $t$. Different from the standard SLB setting in~\eqref{eq:LB-model-assume}, non-stationary linear bandits allow the unknown parameter to vary over time, whose evolution is often called the path-length defined as $P_T = \sum_{t=2}^{T} \norm{\theta_{t-1} - \theta_t}_2$, which naturally measures the non-stationarity of environments. The player's goal is to minimize the following (pseudo) \emph{dynamic regret},
\begin{equation}
  \label{eq:regret-dynamic}
  \textnormal{D-Regret}_T = \sum_{t=1}^{T} \max_{\x \in \X} \x^{\T}\theta_t - \sum_{t=1}^{T} X_t^\T \theta_t,
\end{equation}
namely, the cumulative regret against the optimal strategy that has full information of unknown parameters.

Recently,~\citet{AISTATS'19:window-LB} first established an $\Omega(d^{2/3}T^{2/3}P_T^{1/3})$ minimax lower bound for the non-stationary linear bandits problem. On the upper bound side,~\citet{AISTATS'19:window-LB} proposed an online UCB-type algorithm called WindowUCB, which is based on the sliding window least square estimator to track the evolving parameters. Subsequently,~\citet{NIPS'19:weighted-LB} developed the WeightUCB algorithm, which adopted the weighted least square estimator for parameter estimation. The authors prove an $\Ot(d^{2/3}T^{2/3}P_T^{1/3})$ dynamic regret guarantee for their algorithms, matching the aforementioned lower bound up to $\log T$ factors. However, we exhibit that a serious technical flaw makes their arguments and regret guarantees ungrounded. We revisit the analysis and demonstrate that it is actually impossible to upper bound the crucial term in their argument by the desired quantity as they expected. Further, we present a fix. Without modifying original algorithms, we prove that their algorithms attain an $\Ot(d^{7/8}T^{3/4}P_T^{1/4})$ dynamic regret. This is the first contribution of this paper.

Furthermore, although these two strategies~\citep{AISTATS'19:window-LB,NIPS'19:weighted-LB} attain nice dynamic regret guarantees (after fixing the technical flaws), their algorithms and analyses are fairly complicated. Instead, we discover that a quite simple algorithm based on the \emph{restarted strategy} (simply running an UCB-style algorithm and restarting it periodically), surprisingly, achieves the same dynamic regret guarantee and is more efficient. This is the second contribution of this paper. Indeed, our proposed algorithm enjoys the following three advantages compared with previous studies.
\begin{itemize}
  \item The proposed algorithm for non-stationary linear bandits is very simple and thus easy to analyze its dynamic regret, only exploiting the standard self-normalized concentration inequality for classical stochastic linear bandits. 
  \item Compared with WindowUCB, the sliding window least square based approach~\citep{AISTATS'19:window-LB}, our approach supports online update and enjoys a one-pass manner \emph{without} storing historical data. Meanwhile, WindowUCB demands an $\O(w)$ memory where $w$ is the window length; by contrast, our approach only requires a \emph{constant} memory. 
  \item Compared with WeightUCB, the weighted least square based approach~\citep{NIPS'19:weighted-LB}, our approach and analysis are much simpler, without involving other complicated deviation results. Additionally, WeightUCB maintains and manipulates the covariance matrix and its variant, and thus takes a longer running time.
\end{itemize}

Overall, our approach is more friendly to the resource-constrained learning scenarios due to its simplicity. In the following, we start with a brief review of related work in Section~\ref{sec:related-work}. Then, Section~\ref{sec:infinite-arm} presents our proposed approach and the theoretical results. Section~\ref{sec:analysis} provides the analysis. We further supply the empirical studies in Section~\ref{sec:experiment} and finally conclude the paper in Section~\ref{sec:conclusion}. Appendix~\ref{appendix:tech-lemmas} supplements technical lemmas.

\section{Related Work}
\label{sec:related-work}
Online learning in non-stationary environments has drawn considerable attention recently, in both full-information and bandits settings. We focus on related work in the bandits setting.

Non-stationary multi-armed bandits problem with abrupt changes was first studied by~\citet{JMLR'02:Auer-linear-bandits}. Denoted by $K$ the number of arms and by $L$ the number of distribution changes,~\citet{JMLR'02:Auer-linear-bandits} proposed \textsc{Exp3.S}, a variant of \textsc{Exp3}, which achieves an $\Ot(\sqrt{KLT})$ regret bound when $L$ is known. The rate is minimax optimal up to $\log T$ factors. Later studies demonstrated that $\Ot(\sqrt{KLT})$ regret is attainable by sliding window and weighted penalty strategies~\citep{ALT'11:switch-MAB}, as well as the restarted strategy~\citep{journal'17:restart-MAB}. All these algorithms require the number of changes $L$ as the input parameter, which is undesired in practice. Recently,~\citet{COLT19:Auer-unknown-bandits} achieved a near-optimal rate $\Ot(\sqrt{KLT})$ without knowing prior knowledge of $L$. On the other hand,~\citet{SS'19:non-stationary-MAB} studied the non-stationary MAB with slowly changing distributions, and proved an $\Ot((K\log K)^{1/3}V_T^{1/3}T^{2/3})$ dynamic regret, where $V_T = \sum_{t=2}^{T} \norm{\boldsymbol{\mu}_t - \boldsymbol{\mu}_{t-1}}_{\infty}$ is the total variation of changes in reward distributions.

Non-stationary linear bandits problem was first studied by~\citet{AISTATS'19:window-LB}. The authors established an $\Omega(d^{2/3}T^{2/3}P_T^{1/3})$ minimax lower bound, and then proposed the WindowUCB algorithm based on the sliding window least square, achieving an $\Ot(d^{7/8}T^{3/4}P_T^{1/4})$ dynamic regret (after fixing the technical gap). Nevertheless, to implement the sliding window least square, WindowUCB needs to store historical data in a buffer. A natural replacement is the weighted least square, which supports online update and enjoys both nice empirical performance and sound theoretical guarantee~\citep{control'93:Guo-ffRLS,TKDE'21:DFOP}. Based on the idea,~\citet{NIPS'19:weighted-LB} proposed the WeightUCB algorithm and proved that the approach attains the same dynamic regret. Nevertheless, both algorithmic design and regret analysis of WeightUCB are fairly complicated. Besides, WeightUCB needs to maintain and manipulate covariance matrix and its variant (in the same scale), which leads to an evidently longer running time. Finally, both WindowUCB and WeightUCB require the unknown quantity $P_T$ as an input. To avoid the limitation,~\citet{AISTATS'19:window-LB} developed the bandits-over-bandits mechanism as a meta algorithm and finally obtained an $\Ot(d^{7/8}T^{3/4}P_T^{1/4})$ parameter-free dynamic regret guarantee.

In this work, we first revisit the analysis of two existing algorithms designed for non-stationary linear bandits in the literature~\citep{AISTATS'19:window-LB,NIPS'19:weighted-LB}. We demonstrate that there exists a technical flaw in the analysis, making the claimed dynamic regret guarantee ungrounded. We present a new analysis to fix the technical gap. Next, we propose a simple algorithm based on the restarted strategy for non-stationary linear bandits and show that the simple algorithm can achieve the same dynamic regret guarantee as existing methods. We note that using the restarted strategy for non-stationary environments is not new, which has been applied in various scenarios, including non-stationary online convex optimization~\citep{OR'15:dynamic-function-VT}, MAB with abrupt changes~\citep{journal'17:restart-MAB}, and MAB with gradual changes~\citep{SS'19:non-stationary-MAB}. However, to the best of our knowledge, our work is the first time to apply the restarted strategy to non-stationary linear bandits.

\section{Our Results}
\label{sec:infinite-arm}
We first introduce the formal problem setup and then present our approach.

\subsection{Problem Setup}
\label{sec:setting-LB}
\paragraph{Setting.} In non-stationary (infinite-armed) linear bandits, at each iteration $t$, let $X_t \in \X \subseteq \R^d$ denote the contextual information of the chosen arm and $r_t$ denote its associated reward, and the model is assumed to be linearly parameterized, i.e., 
\begin{equation}
  \label{eq:model-assume}
  r_t = X_t^\T\theta_t + \eta_t,
\end{equation}
where $\theta_t \in \R^d$ is the unknown parameter and $\eta_t$ is the noise satisfying certain tail condition specified below. As mentioned earlier, to guide the algorithmic design of non-stationary linear bandits, it is natural to employ the following (pseudo) \emph{dynamic regret} as the performance measure:
\begin{equation*}
  \textnormal{D-Regret}_T = \sum_{t=1}^{T} \max_{\x \in \X} \x^{\T}\theta_t - \sum_{t=1}^{T} X_t^{\T}\theta_t,
\end{equation*}
which is the cumulative regret against the optimal strategy that has full information of unknown parameters.

\paragraph{Assumptions.} We assume the noise $\eta_t$ be conditionally $R$-sub-Gaussian with a fixed constant $R>0$. That is, $\E[\eta_t \mid X_{1 : t}, \eta_{1 : t-1}] = 0$, and for any $\lambda \in \R$,
\begin{equation*}
  \E[\exp(\lambda \eta_t) \mid X_{1 : t}, \eta_{1 : t-1}] \leq \exp \left(\frac{\lambda^{2} R^{2}}{2}\right),
\end{equation*}
The feasible set and unknown parameters are assumed to be bounded, i.e., $\forall \x \in \X$, $\norm{\x}_2 \leq L$, and $\norm{\theta_t}_2 \leq S$ holds for all $t \in [T]$. For convenience, we further assume $\inner{\x}{\theta_t} \leq 1$, but we will keep the dependence in $L$ and $S$ for better comprehension of the results.

\subsection{RestartUCB Algorithm}
RestartUCB algorithm has two key ingredients: upper confidence bounds for trading off the exploration and exploitation, and the restarted strategy for handling the non-stationarity of environments. Specifically, our proposed RestartUCB algorithm proceeds in epochs. At each iteration, we first estimate the unknown regression parameter from historical data within the epoch, and then construct upper confidence bounds of the expected reward for selecting the arm. Finally, we periodically restart the algorithm to be resilient to the drift of underlying parameter $\theta_t$. 

In the following, we first specify the estimator used in the RestartUCB algorithm, then investigate its estimate error to construct upper confidence bounds, and finally describe the restarted strategy.

\paragraph{Estimator.} At iteration $t$, we adopt the regularized least square estimator by only exploiting data in the current epoch. More precisely, the estimator $\thetah_t$ is the solution of the following problem:
\begin{equation}
  \label{eq:estimator}
  \min_{\theta}~\lambda \norm{\theta}_2^2 + \sum_{s=t_0}^{t-1}(X_s^\T \theta - r_s)^2,
\end{equation}
where $t_0$ is the starting point of the current epoch, and $\lambda > 0$ is the regularization coefficient. Clearly, $\thetah_t$ admits a closed-form solution as 
\begin{equation}
  \label{eq:close-form}
  \thetah_t=V_{t-1}^{-1}\left(\sum_{s=t_0}^{t-1} r_s X_s \right),
\end{equation}
where $V_{t-1} = \lambda I + \sum_{s=t_0}^{t-1} X_sX_s^\T$. We remark that the estimator~\eqref{eq:close-form} (essentially, both the terms of $V_{t-1}$ and $\sum_{s=t_0}^{t-1} r_s X_s$) can be updated online \emph{without} storing historical data in the memory owing to the restarted strategy. Furthermore, it is known that~\eqref{eq:estimator} can be \emph{exactly} solved by the recursive least square algorithm, whose solution is provably equivalent to the closed-form expression~\eqref{eq:close-form}. This feature can further accelerate our approach in that it saves the computation of the inverse of covariance matrix $V_{t-1}$, which is arguably the most time-consuming step at each iteration.

By contrast,~\citet{AISTATS'19:window-LB} adopted the following sliding window least square estimator, 
\begin{equation}
  \label{eq:sw-close-form}
  \thetah^{\text{sw}}_t=(V^{\text{sw}}_{t-1})^{-1}\Bigg(\sum_{s=1 \vee (t-w)}^{t-1} r_s X_s\Bigg),
\end{equation}
where $V^{\text{sw}}_{t-1} = \lambda I + \sum_{s=1 \vee (t-w)}^{t-1} X_sX_s^\T$ is the covariance matrix formed by historical data in the sliding window and $w > 0$ is the window length. For online update, WindowUCB will remove the oldest data item in the window and then add the new item. So it requires to store the nearest $w$ data items in the memory for future update, resulting in an $\O(w)$ space complexity which cannot be regarded as a constant because the setting of $w$ depends on the time horizon $T$.

\paragraph{Upper Confidence Bounds.} Based on the estimator $\thetah_t$ in~\eqref{eq:close-form}, we further construct upper confidence bounds for the expected reward. To this end, it is required to investigate the estimate error. Inspired by the analysis of WindowUCB~\citep{AISTATS'19:window-LB}, we have the following result.

\begin{myLemma}
\label{lemma:estimate-error}
For any $t \in [T]$ and $\delta \in (0,1)$, with probability at least $1-\delta$, the following holds for all $\x \in \X$,
\begin{equation}
  \label{eq:estimate-error}
  \abs{\x^{\T}(\theta_t - \thetah_t)} \leq L^2\sqrt{\frac{dH}{\lambda}} \sum_{p=t_0}^{t-1} \norm{\theta_p - \theta_{p+1}}_2 + \beta_t  \norm{\x}_{V_{t-1}^{-1}},
\end{equation}
where $H > 0$ is the restarting period (or epoch size), and $\beta_t$ is the radius of confidence region,
\begin{equation}
  \label{eq:confidence-radius}
  \beta_t = \sqrt{\lambda} S + R\sqrt{2\log \frac{1}{\delta} + d\log\left(1 + \frac{(t-t_0)L^2}{\lambda d} \right)}.
\end{equation}
\end{myLemma}

\begin{myRemark}
\label{remark:estimate-error}
The analysis of estimate error serves as the foundation of  designing the UCB-type algorithms. In fact, the pioneering study~\citep{AISTATS'19:window-LB} has studied the estimate error of sliding window least square estimator for non-stationary linear bandits, however, the technical reasoning suffers from some gaps and makes the estimate error bound and the claimed $\Ot(T^{2/3}P_T^{1/3})$ dynamic regret guarantee ungrounded. The flaw appears in a key technical lemma~\citep[Lemma 3]{arXiv'19:window-LB}, and is unfortunately inherited by the later studies including WeightUCB~\citep[Theorem 2]{NIPS'19:weighted-LB}, the early version of this paper~\citep[Lemma 3]{AISTATS'20:restart}, and perturbation based method~\citep[Theorem 7]{UAI'20:kim20a}. In this version, we correct the previous results, at the price of another $\sqrt{dH}$ factor appearing in front of the path-length term comparing to the original result (see Lemma 1 in the early version of our work~\citep{AISTATS'20:restart}). The additional $\sqrt{dH}$ factor in the estimation error will lead to an $\Ot(T^{3/4}P_T^{1/4})$ dynamic regret, which is slightly worse than the original $\Ot(T^{2/3}P_T^{1/3})$ rate.  We present more technical discussions in Section~\ref{sec:revisit}.
\end{myRemark}

The estimate error~\eqref{eq:estimate-error} essentially suggests an upper confidence bound of the expected reward $\x^\T \theta_t$. Hence, we adopt the principle of \emph{optimism in the face of uncertainty}~\citep{JMLR'02:Auer-linear-bandits} and choose the arm that maximizes its upper confidence bound,
\begin{equation}
  \label{eq:select-criteria}
  \begin{split}
    X_t = & \argmax_{\x \in \X} \Big\{ \inner{\x}{\thetah_t} + L^2\sqrt{\frac{dH}{\lambda}}\sum_{p=t_0}^{t-1}\norm{\theta_p - \theta_{p+1}}_2  + \beta_t \norm{\x}_{V_t^{-1}} \Big\} \\
    = & \argmax_{\x \in \X} \big\{ \inner{\x}{\thetah_t} + \beta_t \norm{\x}_{V_t^{-1}} \big\}.
    \end{split}
\end{equation}

So at iteration $t$, the algorithm first solves the estimator based on~\eqref{eq:close-form}, then obtains the confidence radius $\beta_t$ by~\eqref{eq:confidence-radius}, and finally pulls the arm $X_t$ according to the selection criteria~\eqref{eq:select-criteria}. 

\paragraph{Restarted Strategy.} To handle the changes of unknown regression parameters, RestartUCB algorithm proceeds in epochs and restarts the procedure every $H$ iterations, as illustrated in Figure~\ref{figure:restart}. We call the variable $H$ as the restarting period or epoch size, which is the key parameter to trade off between the stability and non-stationarity. In each epoch, RestartUCB performs the UCB-style algorithm as described in the last part. We summarize overall procedures in Algorithm~\ref{alg:restart-UCB}.

\begin{figure}[!h]
    \centering
    \includegraphics[width=0.75\textwidth]{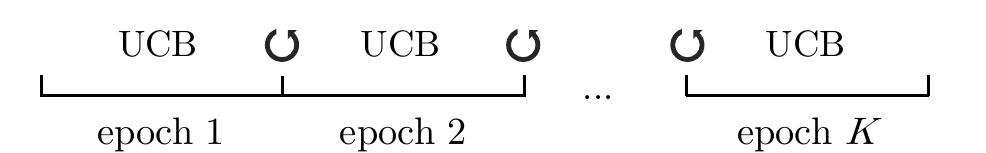}
    \caption{Illustration of RestartUCB algorithm.}
    \label{figure:restart}
\end{figure}

\begin{algorithm}[!t]
   \caption{RestartUCB}
   \label{alg:restart-UCB}
\begin{algorithmic}[1]
\REQUIRE time horizon $T$, restarting period $H$, confidence $\delta$, regularizer $\lambda$, scaling parameters $S$ and $L$
\STATE Set epoch counter $j = 1$
\WHILE{$j \leq \lceil T / H \rceil$}
  \STATE Set $\tau = (j - 1)H$
  \STATE Initialize $X_\tau \in \X$
  \STATE $V_\tau = \lambda I_d$
  \FOR{$t = \tau +1,\ldots,\tau + H - 1$}
    \STATE Compute $\thetah_t = V_{t-1}^{-1}S_{t-1}$
    \STATE Compute $\beta_t$ by~\eqref{eq:confidence-radius} with $t_0 = \tau$
    \STATE Select $X_t = \argmax_{\x \in \X} \{ \inner{\x}{\thetah_t} + \beta_t \norm{\x}_{V_{t-1}^{-1}}\}$
    \STATE Receive the reward $r_t$
    \STATE Update $V_{t} = V_{t-1} + X_t X_t^\T$ and $S_t = S_{t-1} + r_t X_t$
  \ENDFOR
  \STATE Set $j = j + 1$
\ENDWHILE
\end{algorithmic}
\end{algorithm}

We will show in the next subsection that a fixed length is sufficient to achieve the same theoretical guarantees as previous works~\citep{AISTATS'19:window-LB,NIPS'19:weighted-LB}. Nevertheless, since the regret guarantee is not optimal, it would be interesting to see whether an adaptive epoch length with a certain statistical detection will give an improved regret guarantee.

\subsection{Theoretical Guarantees}
\label{sec:infinite-theory}
We show that notwithstanding its simplicity RestartUCB algorithm enjoys the same dynamic regret guarantee as the existing methods for non-stationary linear bandits, including WindowUCB~\citep{AISTATS'19:window-LB} and WeightUCB~\citep{NIPS'19:weighted-LB}. 

In the following, we first analyze the regret within each epoch (Theorem~\ref{thm:dynamic-regret-in-epoch}), and then sum over epochs to obtain the guarantee of the whole time horizon (Theorem~\ref{thm:dynamic-regret}).

\begin{myThm}
\label{thm:dynamic-regret-in-epoch}
For each epoch $\mathcal{E}$ whose size is $H$ and any $\delta \in (0,1)$, with probability at least $1-2\delta$, the dynamic regret within the epoch is upper bounded by 
\begin{equation*}
\begin{split}
  \DReg(\mathcal{E}) \triangleq {} &  \sum_{t \in \mathcal{E}} \max_{\x \in \X} \x^{\T}\theta_t - \sum_{t \in \mathcal{E}} X_t^\T \theta_t \\
  \leq {} & 2L^2\sqrt{\frac{d}{\lambda}}\cdot H^{\frac{3}{2}}\mathcal{P}(\Ecal) + 2\beta_H \sqrt{2dH\log\left(1+ \frac{L^2 H}{\lambda d}\right)},  
\end{split}
\end{equation*}
where $\beta_H = \sqrt{\lambda} S + R\sqrt{2\log(1/\delta) + d\log\left(1 + \frac{HL^2}{\lambda d} \right)}$ is the confidence radius of the epoch, and $\mathcal{P}(\mathcal{E})$ denotes the path-length within epoch $\mathcal{E}$, i.e., $\mathcal{P}(\mathcal{E}) = \sum_{t\in \mathcal{E}} \norm{\theta_{t-1} - \theta_t}_2$.
\end{myThm}

By summing the dynamic regret over epochs, we can therefore obtain dynamic regret over of the whole time horizon.
\begin{myThm}
\label{thm:dynamic-regret}
With probability at least $1-1/T$, the dynamic regret of RestartUCB (Algorithm~\ref{alg:restart-UCB}) over the whole time horizon is upper bounded by 
\begin{equation}
    \label{eq:main-result}
    \DReg_T = \sum_{t=1}^{T} \max_{\x \in \X} \x^{\T}\theta_t - \sum_{t=1}^{T} X_t^\T \theta_t \leq \Ot\left(d^{\frac{1}{2}}H^{\frac{3}{2}} P_T + dT/\sqrt{H}\right),
\end{equation}
where $P_T = \sum_{t=2}^{T} \norm{\theta_{t-1} - \theta_t}_2$ is the path-length, and $H$ is the restarting period.

Furthermore, by setting the restarting period optimally as
\begin{equation}
    \label{eq:optimal-tuning-parameter}
    H = \min\{H^*, T\} = \min\left\{ \floor{d^{\frac{1}{4}}T^{\frac{1}{2}} P_T^{-\frac{1}{2}}}, T \right\},
\end{equation}
RestartUCB achieves the following dynamic regret,
\begin{equation}
    \label{eq:optimal-tuning}
    \DReg_T \leq
    \begin{cases}
    \Ot\big(d^{\frac{7}{8}} T^{\frac{3}{4}} P_T^{\frac{1}{4}}\big) & \mbox{ when } P_T \geq \sqrt{d}/T,\vspace{2mm}\\
    \Ot(d\sqrt{T}) & \mbox{ when } P_T < \sqrt{d}/T.
    \end{cases} 
\end{equation}
\end{myThm}

\begin{myRemark}
As shown in Theorem~\ref{thm:dynamic-regret}, the setting of optimal restarting period $H^*$ in~\eqref{eq:optimal-tuning-parameter} requires the prior information of path-length $P_T$, which is generally unavailable. We will discuss how to remove the undesired dependence in the next subsection.
\end{myRemark}

\subsection{Adapting to Unknown Non-stationarity}
\label{sec:unknown-path-length}
As mentioned earlier, the restarting period $H$ plays a key role in dealing with the non-stationarity of environments. Intuitively, a small restarting period should be employed when the environments change very dramatically, and a large one should be used when the environments are relatively stable. Our theoretical result also validates the intuition. As one can see in Theorem~\ref{thm:dynamic-regret}, the restarting period $H$ trades off between the path-length term $P_T$ and the total horizon $T$, which necessitates an appropriate balance. Nevertheless, the optimal configuration of restarting period, as shown in~\eqref{eq:optimal-tuning-parameter}, requires the prior knowledge of path-length $P_T$, which essentially measures the non-stationarity of underlying environments and is thus generally unavailable. 

To compensate the lack of this information of environmental non-stationarity, we design an online ensemble method~\citep{book'12:ensemble-zhou} by employing the meta-base framework, which is recently used in full-information non-stationarity online learning~\citep{ICML'15:Daniely-adaptive,AISTATS'17:coin-betting-adaptive,NIPS'18:Zhang-Ader,NIPS'19:Zheng,AISTATS'20:Zhang,UAI'20:simple,NIPS'20:sword} and non-stationary bandit online learning~\citep{COLT'17:Corralling-bandits,AISTATS'19:window-LB,AISTATS'20:BCO}. In this paper, to deal with the issue for non-stationary linear bandits, we employ the \emph{Bandits-over-Bandits} (BOB) mechanism, proposed by~\citet{AISTATS'19:window-LB} in designing parameter-free algorithm for non-stationary linear bandits based on sliding window least square estimator. 

\begin{figure}[!t]
    \centering
    \includegraphics[width=0.85\textwidth]{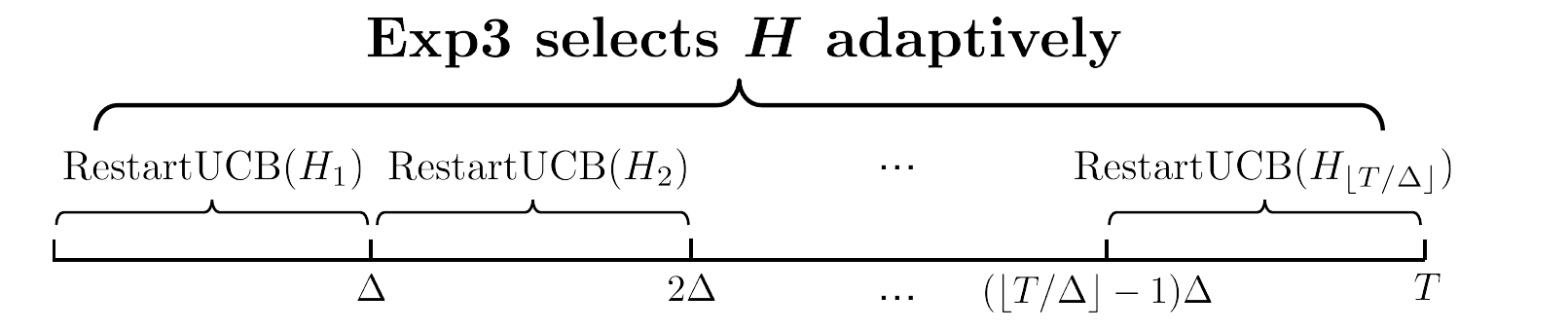}
    \caption{Illustration of Bandits-over-Bandits mechanism~\citep{arXiv'19:window-LB} with application to our proposed RestartUCB algorithm. The overall algorithm is in two-layer meta-base structure: the algorithm treats RestartUCB as the base-learner to handle non-stationary linear bandits with a given restarting period, and employs \textsc{Exp3} as the meta-learner to adaptively choose the optimal restarting period $H$. }
    \label{figure:restartBOB}
\end{figure}

In the following, we describe how to apply the BOB mechanism to eliminate the requirement of the unknown path-length in RestartUCB. The essential idea is to use the RestartUCB algorithm as the base-learner to handle non-stationary linear bandits with a given restarting period, and on top of that we will employ a second-layer bandits algorithm as the meta-learner to adaptively learn the optimal restarting period. We name the RestartUCB algorithm with Bandits-over-Bandits mechanism as ``RestartUCB-BOB''. Figure~\ref{figure:restartBOB} illustrates the meta-base structure of RestartUCB-BOB algorithm. To be more concrete, although the exact value of the optimal restarting period $H^*$ (or equivalently, the path-length $P_T$) is unknown, we can make some guess of its possible value, as the $P_T$ is always bounded. Then, we can use a certain meta-algorithm to adaptively track the best restarting period. To achieve this goal, RestartUCB-BOB first requires to examine the performance of base-learner with different restarting period. Therefore, RestartUCB-BOB will perform in several episodes, and in each episode RestartUCB-BOB employs the base-algorithm RestartUCB with a particular restarting period and receives the returned cumulative reward over the episode as the reward feedback. We denote by $\Delta \in [T]$ the episode length. The restarting period will be adaptively adjusted by employing \textsc{Exp3}~\citep{SICOMP'02:Auer-EXP3} as the meta-algorithm. In the configuration of RestartUCB-BOB, the episode length is set $\Delta = \ceil{d\sqrt{T}}$. Further, the pool of candidate restarting periods $\H$ is configured as follows:
\begin{equation}
  \label{eq:candidate-pool}
  \H = \left\{H_i = \floor{ d^{\frac{1}{4}} S^{-\frac{1}{2}} \cdot 2^{i-1}} \mid i \in [N]\right\},
\end{equation}
where $N = \ceil{\frac{1}{2}\log_2(ST)} + 1$ is the number of candidate restarting periods and recall that $S$ is the upper bound of the norm of underlying regression parameters as specified in Section~\ref{sec:setting-LB}. Let $H_{\min}$ ($H_{\max}$) be the minimal (maximal) restarting period in the pool $\H$, then it is evident to verify that
\begin{equation}
  \label{eq:min-max-epoch-size}
  H_{\min} = \floor{d^{\frac{1}{4}} S^{-\frac{1}{2}}}, H_{\max} = \floor{d^{\frac{1}{4}}\sqrt{T}} \leq \Delta. 
\end{equation}
To conclude, RestartUCB-BOB is in a two-layer meta-base structure and will perform in episodes. In each episode, the base-learner is RestartUCB associated with a particular restarting period in the candidate pool determined by the meta-learner \textsc{Exp3}; besides, the cumulative reward of the base-learner within the episode is fed to the meta-learner as the feedback to adaptively choose a better restarting period. We refer the reader to Section 7 of~\citet{arXiv'19:window-LB} for more descriptions of algorithmic details.

The following theorem presents the dynamic regret guarantee for RestartUCB-BOB. Note that the algorithm does not require the prior knowledge of the path-length $P_T$.
\begin{myThm}
\label{thm:dynamic-regret-BOB}
\textsc{RestartUCB} together with Bandits-over-Bandits mechanism satisfies
\begin{equation}
  \label{eq:dynamic-order-infinite}
  \textnormal{D-Regret}_T = \sum_{t=1}^{T} \max_{\x \in \X} \x^{\T}\theta_t - \sum_{t=1}^{T} X_t^\T \theta_t \leq \Ot\big(d^{\frac{7}{8}} T^{\frac{3}{4}} P_T^{\frac{1}{4}}\big),
\end{equation}
without requiring the path-length $P_T$ ahead of time.
\end{myThm}

\begin{myRemark}
From the theorem, we can observe that RestartUCB-BOB enjoys the same dynamic regret bound as RestartUCB with an oracle tuning~\eqref{eq:optimal-tuning} shown in Theorem~\ref{thm:dynamic-regret}, while RestartUCB-BOB now requires no prior knowledge on the environmental non-stationarity measure $P_T$. Nevertheless, the attained dynamic regret upper bound still exhibits a certain gap to the $\Omega(d^{2/3}T^{2/3}P_T^{1/3})$ minimax lower bound of non-stationary linear bandits~\citep{AISTATS'19:window-LB}. Therefore, it remains open on how to obtain rate-optimal and parameter-free dynamic regret. Indeed, even with an oracle tuning, RestartUCB still cannot achieve optimal dynamic regret. We are not sure whether this is due to the limitation of the regret analysis or the algorithm itself. Finally, we note that recent studies achieve near-optimal dynamic regret without prior information for multi-armed bandits~\citep{EWRL'18:Auer-MAB,COLT19:Auer-unknown-bandits} and contextual bandits~\citep{COLT'19:para-free-contextual} by means of change detection. These studies could be useful in designing parameter-free algorithms for non-stationary linear bandits, which will be investigated in the future.
\end{myRemark}

\section{Analysis}
\label{sec:analysis}
In this section, we provide proofs of theoretical results presented in the previous section.
\subsection{Proof of Lemma~\ref{lemma:estimate-error}}
\begin{proof}
From the model assumption~\eqref{eq:model-assume} and the estimator~\eqref{eq:close-form}, we can verify that the estimate error can be decomposed as, 
\begin{equation*}
    \thetah_t - \theta_t = V_{t-1}^{-1}\Bigg(\sum_{s=t_0}^{t-1} X_s X_s^\T (\theta_s - \theta_t) + \sum_{s=t_0}^{t-1} \eta_s X_s - \lambda \theta_t\Bigg).
\end{equation*}
Therefore, by Cauchy-Schwarz inequality, we know that for any $\x \in \X$, 
\begin{equation}
  \label{eq:bound-cauchy}
  \abs{\x^{\T}(\thetah_t - \theta_t)} \leq \norm{\x}_2\cdot A_t + \norm{\x}_{V_{t-1}^{-1}}\cdot B_t,
\end{equation}
where 
\begin{align*}
A_t = \left\| V_{t-1}^{-1}\left(\sum_{s=t_0}^{t-1} X_s X_s^\T (\theta_s - \theta_t)\right)\right\|_2, \mbox{  and } B_t = \left\|\sum_{s=t_0}^{t-1} \eta_s X_s - \lambda \theta_t\right\|_{V_{t-1}^{-1}}.
\end{align*}

We will give upper bounds for these two terms separately, as summarized in the following two lemmas.
\begin{myLemma}
\label{lemma:path-length-A_t}
For any $t \in [T]$, we have 
\begin{equation}
    \label{eq:path-length-A_t}
    \left\| V_{t-1}^{-1}\bigg(\sum_{s=t_0}^{t-1} X_s X_s^\T (\theta_s - \theta_t)\bigg)\right\|_2 \leq L \sqrt{\frac{dH}{\lambda}}\sum_{p=t_0}^{t-1} \norm{\theta_p - \theta_{p+1}}_2.
\end{equation}
\end{myLemma}

\begin{myLemma}
\label{lemma:path-length-B_t}
For any $t \in [T]$, we have 
\begin{equation}
    \left\|\sum_{s=t_0}^{t-1} \eta_s X_s - \lambda \theta_t\right\|_{V_{t-1}^{-1}} \leq \sqrt{\lambda} S + R\sqrt{2\log \frac{1}{\delta} + d\log\left(1 + \frac{(t-t_0)L^2}{\lambda d} \right)},
\end{equation}
where $\beta_t \triangleq \sqrt{\lambda} S + R\sqrt{2\log \frac{1}{\delta} + d\log\left(1 + \frac{(t-t_0)L^2}{\lambda d} \right)}$ is the confidence radius used in RestartUCB.
\end{myLemma}

Based on the inequality~\eqref{eq:bound-cauchy}, Lemma~\ref{lemma:path-length-A_t}, Lemma~\ref{lemma:path-length-B_t}, and the boundedness of the feasible set, we have for any $\x \in \X$,
\[
  \abs{\x^{\T}(\theta_t - \thetah_t)} \leq L^2\sqrt{\frac{dH}{\lambda}} \sum_{p=t_0}^{t-1} \norm{\theta_p - \theta_{p+1}}_2 + \beta_t  \norm{\x}_{V_{t-1}^{-1}},
\]
which competes the proof.
\end{proof}

We proceed to prove Lemma~\ref{lemma:path-length-A_t} and Lemma~\ref{lemma:path-length-B_t}. It is noteworthy mentioning that previous works in non-stationary linear bandits~\citep{arXiv'19:window-LB,NIPS'19:weighted-LB,AISTATS'20:restart} also need to upper bound some quantities similar to $A_t$, while their results are general invalid due to a serious technical flaw that will be explicitly stated in Section~\ref{sec:revisit}. Lemma~\ref{lemma:path-length-A_t} serves as the key component to fix existing results. 

\begin{proof}[{Proof of Lemma~\ref{lemma:path-length-A_t}}]
Notice that 
\begin{align}
\left\|V_{t-1}^{-1} \bigg(\sum_{s=t_0}^{t-1} X_s X_s^\T (\theta_s - \theta_t)\bigg)\right\|_2 = {} & \left\|V_{t-1}^{-1} \bigg(\sum_{s=t_0}^{t-1} X_s X_s^\T \Big(\sum_{p=s}^{t-1} (\theta_p - \theta_{p+1})\Big)\bigg)\right\|_2 \nonumber \\
  = {} & \left\|V_{t-1}^{-1} \bigg(\sum_{p=t_0}^{t-1} \Big(\sum_{s=t_0}^{p} X_s X_s^\T (\theta_p - \theta_{p+1})\Big)\bigg)\right\|_2 \nonumber \\
  \leq {} & \sum_{p=t_0}^{t-1} \left\|V_{t-1}^{-1} \Big(\sum_{s=t_0}^{p} X_s X_s^\T\Big) (\theta_p - \theta_{p+1})\right\|_2 \nonumber \\
  \leq {} & \sum_{p=t_0}^{t-1} \left\|V_{t-1}^{-1} \Big(\sum_{s=t_0}^{p} X_s X_s^\T\Big)\right\|_2 \norm{\theta_p - \theta_{p+1}}_2.\nonumber
\end{align}

We now derive the upper bound for the term $\|V_{t-1}^{-1} (\sum_{s=t_0}^{p} X_s X_s^\T)\|_2$. Denote by $\S(1) = \{\x \mid \norm{\x}_2 = 1\}$ the unit sphere.
\begin{align*}
   \left \| V_{t-1}^{-1} \Big(\sum_{s=t_0}^{p} X_s X_s^\T\Big) \right \|_2 = {} & \sup_{\z \in \S(1)} \sup_{\tilde{\z} \in \S(1)} \left| \z^\T V_{t-1}^{-1} \Big(\sum_{s=t_0}^{p} X_s X_s^\T\Big) \tilde{\z} \right| \\  
   = {} & \left| \z_*^\T V_{t-1}^{-1} \Big(\sum_{s=t_0}^{p} X_s X_s^\T\Big) \tilde{\z}_* \right| \\
  \leq {} & \norm{\z_*}_{V_{t-1}^{-1}} \left\| \sum_{s=t_0}^{p} X_s (X_s^\T \tilde{\z}_*)\right\|_{V_{t-1}^{-1}}\\
  \leq {} & \norm{\z_*}_{V_{t-1}^{-1}} \left\| \sum_{s=t_0}^{p} X_s \norm{X_s}_2 \norm{\tilde{\z}_*}_2\right\|_{V_{t-1}^{-1}}\\
  \leq {} & \frac{L}{\sqrt{\lambda}} \left\| \sum_{s=t_0}^{p} X_s\right\|_{V_{t-1}^{-1}}\\
  \leq {} & \frac{L}{\sqrt{\lambda}} \sum_{s=t_0}^{p} \norm{X_s}_{V_{t-1}^{-1}} \\
  \leq {} & \frac{L}{\sqrt{\lambda}} \sqrt{H} \sqrt{\sum_{s=t_0}^{p} \norm{X_s}_{V_{t-1}^{-1}}^2} \\ 
  \leq {} & L \sqrt{\frac{dH}{\lambda}}.
\end{align*} 
In above, the first equation makes use of the property of the matrix $2$-norm: for a matrix $M \in \R^{m \times n}$, $\norm{M}_2 = \sup_{\norm{\x}_2 = 1} \sup_{\norm{\y}_2 = 1} \abs{\x^\T M \y}$, whose proof can be found from the book~\citep[Chapter 5, Eq.~(5.2.9)]{meyer2000matrix} and also Lemma~\ref{lemma:matrix-2-norm} for self-containedness. Further, $(\z_*,\tilde{\z}_*)$ denotes the optimizer of the right hand side optimization problem in the first line. In the proof, we use the fact that for any $\x$, we have $\norm{\x}_{V_{t-1}^{-1}} \leq \norm{\x}_2/\sqrt{\lambda} $ as $V_{t-1} \succeq \lambda I$. The second last step holds by the Cauchy-Schwarz inequality. Besides, the last step follows from the fact: for any $p \in \{t_0,\ldots,t-1\}$,
\begin{align*} 
    {} & \sum_{s=t_0}^{p} \norm{X_s}_{V_{t-1}^{-1}}^2 = \sum_{s=t_0}^{p} \mathrm{Tr}(X_s^\T V_{t-1}^{-1} X_s) \\
  = {} & \mathrm{Tr}\left(V_{t-1}^{-1} \sum_{s=t_0}^{p} X_s X_s^\T\right) \\
  \leq {} & \mathrm{Tr}\left(V_{t-1}^{-1} \sum_{s=t_0}^{p} X_s X_s^{\T}\right) + \sum_{s=p+1}^{t-1}X_s^{\T}V_{t-1}^{-1} X_s + \lambda \sum_{i=1}^{d} \mathbf{e}_i^{\T} V_{t-1}^{-1} \mathbf{e}_i \\
  = {} & \mathrm{Tr}\left(V_{t-1}^{-1} \sum_{s=t_0}^{p} X_s X_s^{\T}\right) + \mathrm{Tr}\left(V_{t-1}^{-1} \sum_{s=p+1}^{t-1} X_s X_s^{\T}\right) + \mathrm{Tr}\left(V_{t-1}^{-1} \lambda \sum_{i=1}^{d} \mathbf{e}_i \mathbf{e}_i^{\T}\right) \\
  = {} & \mathrm{Tr}(I_d) = d.
\end{align*} 
Hence, we complete the proof.
\end{proof}

\begin{proof}[{Proof of Lemma~\ref{lemma:path-length-B_t}}]
From the self-normalized concentration inequality~\citep[Theorem 1]{NIPS'11:AY-linear-bandits}, restated in Theorem~\ref{thm:self-normalize} of Section~\ref{appendix:tech-lemmas}, we know that
\begin{align*}
  \left\| \sum_{s=t_0}^{t-1} \eta_s X_s \right\|_{V_{t-1}^{-1}} \overset{\eqref{eq:potential}}{\leq} {} & \sqrt{2R^2 \log\left(\frac{\det(V_{t-1})^{1/2} \det(\lambda I)^{-1/2}}{\delta}\right)}\\
  \leq {} & R\sqrt{2\log\frac{1}{\delta} + d\log\left(1 + \frac{(t-t_0)L^2}{d}\right)},
\end{align*}
where the last inequality is obtained from the analysis of the determinant, as shown in the proof of Lemma~\ref{lemma:potential}. Meanwhile, since $V_{t-1} \succeq \lambda I$, we know that 
\begin{align*}
  \norm{\lambda \theta_t}_{V^{-1}_{t-1}}^2 \leq & 1/\lambda_{\min}(V_{t-1}) \norm{\lambda \theta_t}_2^2 \leq \frac{1}{\lambda} \norm{\lambda \theta_t}_2^2 \leq \lambda S^2.
\end{align*}
Therefore, the upper bound of $B_t$ can be immediately obtained by  combining the above inequalities.
\end{proof}

\subsection{Proof of Theorems~\ref{thm:dynamic-regret-in-epoch} and~\ref{thm:dynamic-regret}}

\begin{proof}[Proof of Theorem~\ref{thm:dynamic-regret-in-epoch}]
Due to Lemma~\ref{lemma:estimate-error} and the fact that $X_t^*, X_t \in \X$, each of the following holds with probability at least $1-\delta$, 
\begin{align*}
  \inner{X_t^*}{\theta_t} \leq {}& \inner{X_t^*}{\thetah_t} + L^2\sqrt{\frac{dH}{\lambda}}\sum_{p=t_0}^{t-1} \norm{\theta_p - \theta_{p+1}}_2 + \beta_t  \norm{X_t^*}_{V_{t-1}^{-1}},\\
  \inner{X_t}{\theta_t} \geq {}& \inner{X_t}{\thetah_t} - L^2\sqrt{\frac{dH}{\lambda}}\sum_{p=t_0}^{t-1} \norm{\theta_p - \theta_{p+1}}_2 - \beta_t  \norm{X_t}_{V_{t-1}^{-1}}.
\end{align*}
By the union bound, the following holds with probability at least $1-2\delta$,
\begin{align*}
  {} & \langle X_t^*,\theta_t\rangle - \langle X_t,\theta_t\rangle\\
  \leq {} & \inner{X_t^*}{\thetah_t} - \inner{X_t}{\thetah_t} + 2L^2\sqrt{\frac{dH}{\lambda}}\sum_{p=t_0}^{t-1} \norm{\theta_p - \theta_{p+1}}_2 + \beta_t (\norm{X_t^*}_{V_{t-1}^{-1}} + \norm{X_t}_{V_{t-1}^{-1}})\\
  \leq {} & 2L^2\sqrt{\frac{dH}{\lambda}}\sum_{p=t_0}^{t-1} \norm{\theta_p - \theta_{p+1}}_2 + 2 \beta_t \norm{X_t}_{V_{t-1}^{-1}},
\end{align*}
where the last step comes from the following implication of the arm selection criteria~\eqref{eq:select-criteria} such that $\inner{X_t^*}{\thetah_t} + \beta_t \norm{X_t^*}_{V_{t-1}^{-1}} \leq \inner{X_t}{\thetah_t} + \beta_t \norm{X_t}_{V_{t-1}^{-1}}$.

Hence, dynamic regret within epoch $\Ecal$ is bounded by,
\begin{align*}
 \textnormal{D-Regret}(\Ecal) \leq {} & \sum_{t\in\Ecal} 2L^2\sqrt{\frac{dH}{\lambda}}\sum_{p=t_0}^{t-1} \norm{\theta_p - \theta_{p+1}}_2 + 2 \beta_t \norm{X_t}_{V_{t-1}^{-1}}\\
 \leq {} & 2L^2\sqrt{\frac{d}{\lambda}}H^{\frac{3}{2}}\mathcal{P}(\Ecal) + 2\beta_H \sqrt{2dH\log\left(1+ \frac{L^2 H}{\lambda d}\right)},
\end{align*}
where the last inequality holds due to the standard elliptical potential lemma (Lemma~\ref{lemma:potential}), whose statement and proof are presented in Section~\ref{appendix:tech-lemmas}.
\end{proof}

\begin{proof}[Proof of Theorem~\ref{thm:dynamic-regret}] By taking the union bound over the dynamic regret of all $\ceil{T/H}$ epochs, we know that the following holds with probability at least $1-2/T$,
\begin{align*}
  \textnormal{D-Regret}_T =  \sum_{s=1}^{\ceil{T/H}}\textnormal{D-Regret}(\mathcal{E}_s) \leq 2L^2\sqrt{\frac{d}{\lambda}}H^{\frac{3}{2}} P_T + 2T\tilde{\beta}_H\sqrt{\frac{2d}{H}\log\left(1+ \frac{L^2 H}{\lambda d}\right)},
\end{align*}
where $\tilde{\beta}_H = \sqrt{\lambda} S + R\sqrt{2\log(T\ceil{\frac{T}{H}}) + d\log\left(1 + \frac{HL^2}{\lambda d} \right)}$. Ignoring logarithmic factors in time horizon $T$, we finally obtain that 
\[
  \textnormal{D-Regret}_T \leq \Ot\big(d^{\frac{1}{2}}H^{\frac{3}{2}}P_T + dT/\sqrt{H}\big).
\]
When $P_T < \sqrt{d}/T$ (which corresponds a small amount of non-stationarity), we simply set the restarting period as $T$ and achieve an $\Ot(d\sqrt{T})$ regret bound. Note that under such a configuration, our algorithm actually performs no restart and thereby recovers the standard LinUCB algorithm for the stationary stochastic linear bandits~\citep{NIPS'11:AY-linear-bandits}. Besides, when coming to the non-degenerated case of $P_T \geq \sqrt{d}/T$, we set the restarting period optimally as $H = \floor{d^{1/4}T^{1/2} P_T^{-1/2}}$ and attain an $\Ot(d^{\frac{7}{8}} T^{\frac{3}{4}} P_T^{\frac{1}{4}})$ dynamic regret bound. This ends the proof. 
\end{proof}

\subsection{Revisiting Existing Results}
\label{sec:revisit}
Previous studies show an $\Ot(T^{2/3}P_T^{1/3})$ dynamic regret for non-stationary linear bandits, however, the technical reasoning suffers from some gaps and makes the overall regret guarantee ungrounded. In the following, we first spot the flaws of their original proofs and then discuss the key component of our new analysis.

Indeed, the flaw appears in a key technical lemma for regret analysis of WindowUCB, the pioneering study of non-stationary linear bandits~\citep[Lemma 3]{arXiv'19:window-LB}. The flaw is unfortunately inherited by the later studies, including WeightUCB~\citep[Theorem 2]{NIPS'19:weighted-LB}, the early version of this paper~\citep[Lemma 3]{AISTATS'20:restart}, and perturbation based method~\citep[Theorem 7]{UAI'20:kim20a}. To be more concrete, Lemma 3 of~\citet{arXiv'19:window-LB} (also see Lemma 3 of~\citet{AISTATS'20:restart}) claims that for any $t \in [T]$, 
\begin{equation}
  \label{eq:claim}
  \left\| V_{t-1}^{-1}\bigg(\sum_{s=t_0}^{t-1} X_s X_s^\T (\theta_s - \theta_t)\bigg)\right\|_2 \leq \sum_{p=t_0}^{t-1} \norm{\theta_p - \theta_{p+1}}_2.
\end{equation}
Actually, the quantity $\| V_{t-1}^{-1}(\sum_{s=t_0}^{t-1} X_s X_s^\T (\theta_s - \theta_t))\|_2$ is of great importance for the regret analysis of non-stationary linear bandits algorithms, because it will be finally converted to the path-length of unknown regression parameters. Our Lemma~\ref{lemma:path-length-A_t} gives an upper bound of $L \sqrt{\frac{dH}{\lambda}}\sum_{p=t_0}^{t-1} \norm{\theta_p - \theta_{p+1}}_2$, which has a worse dependence in terms of dimension $d$ and restarting period $H$ compared to~\eqref{eq:claim}. However, we will demonstrate that the proof of the above claim~\eqref{eq:claim} suffers serious technical flaws, which makes the result ungrounded. We restate their proof~\citep[Appendix B]{arXiv'19:window-LB} as follows:
\begin{align}
    {} & \left\|V_{t-1}^{-1} \bigg(\sum_{s=t_0}^{t-1} X_s X_s^\T (\theta_s - \theta_t)\bigg)\right\|_2 \nonumber \\
  = {} & \left\|V_{t-1}^{-1} \bigg(\sum_{s=t_0}^{t-1} X_s X_s^\T \Big(\sum_{p=s}^{t-1} (\theta_p - \theta_{p+1})\Big)\bigg)\right\|_2 \nonumber \\
  = {} & \left\|V_{t-1}^{-1} \bigg(\sum_{p=t_0}^{t-1} \Big(\sum_{s=t_0}^{p} X_s X_s^\T (\theta_p - \theta_{p+1})\Big)\bigg)\right\|_2 \nonumber \\
  \leq {} & \sum_{p=t_0}^{t-1} \left\|V_{t-1}^{-1} \Big(\sum_{s=t_0}^{p} X_s X_s^\T\Big) (\theta_p - \theta_{p+1})\right\|_2 \nonumber \\
  \leq {} & \sum_{p=t_0}^{t-1} \sigma_{\max}\left(V_{t-1}^{-1} \Big(\sum_{s=t_0}^{p} X_s X_s^\T\Big)\right) \norm{\theta_p - \theta_{p+1}}_2 \nonumber\\
  \leq {} & \sum_{p=t_0}^{t-1} \norm{\theta_p - \theta_{p+1}}_2, \label{eq:step}
\end{align}
where $\sigma_{\max} (\cdot)$ is the largest singular value. The key is the last step~\eqref{eq:step} but its proof is questionable: they need to show the following results holds universally for all $p \in \{t_0,\ldots,t-1\}$,
\begin{equation}
  \label{eq:key-to-hold}
  \sigma_{\max}\left(V_{t-1}^{-1} \Big(\sum_{s=t_0}^{p} X_s X_s^\T\Big)\right) \leq 1.
\end{equation}
To this end, denoted by $A = \sum_{s=t_0}^{p} X_s X_s^\T$, the authors show that $V_{t-1}^{-1} A$ shares the same characteristics polynomial with $V_{t-1}^{-1/2} A V_{t-1}^{-1/2}$, namely, $\det(\eta I - V_{t-1}^{-1} A) = \det(\eta I - V_{t-1}^{-1/2} A V_{t-1}^{-1/2})$ holds for any $\eta$. Since $V_{t-1}^{-1/2} A V_{t-1}^{-1/2}$ is clearly symmetric positive semi-definite, they claim that 
\begin{equation}
    \label{eq:wrong-claim}
    \z^{\T} V_{t-1}^{-1}A \z \geq 0
\end{equation}
also holds for $\z \in \S(1) =\{\x \mid \norm{\x}_2 = 1\}$, which is crucial for their remaining proof. 
\begin{align}
& \sigma_{\max}\left(V_{t-1}^{-1}\bigg(\sum_{s=t_0}^{p} X_{s} X_{s}^\T\bigg)\right)=\sup_{\z \in \S(1)} \z^\T V_{t-1}^{-1}\left(\sum_{s=t_0}^{p} X_{s} X_{s}^\T\right) \z \label{eq:wrong-1}\\
\overset{\eqref{eq:wrong-claim}}{\leq} & \sup_{\z \in \S(1)}\left\{\z^\T V_{t-1}^{-1}\bigg(\sum_{s=t_0}^{p} X_{s} X_{s}^\T\bigg) \z+\z^\T V_{t-1}^{-1}\bigg(\sum_{s=p+1}^{t-1} X_{s} X_{s}^\T\bigg) \z+\lambda \z^\T V_{t-1}^{-1} \z\right\} \label{eq:wrong-2}\\
=& \sup_{\z \in \S(1)} \z^\T V_{t-1}^{-1} V_{t-1} \z=1. \nonumber
\end{align}
However, we identify that there are two issues in the above arguments. First, the step in~\eqref{eq:wrong-1} doubtful. For a matrix $M \in \R^{m \times n}$, we have $\norm{M}_2 = \sup_{\norm{\x}_2 = 1} \sup_{\norm{\y}_2 = 1} \abs{\y^{\T} M \x}$ (see for Lemma~\ref{lemma:matrix-2-norm}), while it is not warranted that $\norm{M}_2 = \sup_{\norm{\z}_2 = 1} \abs{\z^{\T} M \z}$ which is seemingly important for the following arguments. Regardless of this first issue, the second issue about the claim~\eqref{eq:wrong-claim} and the result in~\eqref{eq:wrong-2}. We discover that the claim~\eqref{eq:wrong-claim} is even more severe. We discover that the claim~\eqref{eq:wrong-claim} is ungrounded (at least its current proof cannot stand for the correctness). The big caveat is that $V_{t-1}^{-1}A \in \R^{d\times d}$ is \emph{not} guaranteed to be symmetric. The logic behind the claim is that, suppose $P,Q\in\R^{d\times d}$ are with the same characteristics polynomial, i.e., $\det(\eta I - Q) = \det(\eta I - P)$ holds for any $\eta$, and meanwhile $P$ is symmetric positive semi-definite (which guarantees $\z^\T P \z \geq 0$ for any $\z \in \R^d$), then we can also have $\z^\T Q \z \geq 0$ for any $\z \in \R^d$. Unfortunately, the reasoning is not correct, and we give a simple counterexample. Let $P$ be the $2$-dim identity matrix $[1,0;0,1]$, and $Q = [1,-10;0,1]$ is an \emph{asymmetric} matrix, then clearly $\det(\eta I - P)=\det(\eta I - Q) = (\eta - 1)^2$ is true for any $\eta$; however, $\z^\T Q \z \geq 0$ does not hold in general, for example, $\z^\T Q \z = -8 < 0$ when $\z = (1,1)^{\T}$.

\paragraph{Fixing the gap.} The term in the left hand of~\eqref{eq:claim} is crucial for the dynamic regret analysis of non-stationary linear bandits, which is expected to be converted to the path-length indicating the non-stationarity of environments. In the proof of~\citet{AISTATS'19:window-LB}, as shown in~\eqref{eq:step} and~\eqref{eq:key-to-hold}, the authors aim to upper bound the crucial quantity $\sigma_{\max}\left(V_{t-1}^{-1} \big(\sum_{s=t_0}^{p} X_s X_s^\T\big)\right)$ by some constant. However, the technical reasoning is wrong. We avoid this issue and provide a new analysis as exhibited in the proof of Lemma~\ref{lemma:path-length-A_t}, which serves as the key component in our fix. Following our analysis of Lemma~\ref{lemma:path-length-A_t}, it is not hard to give a similar estimator error analysis for the sliding window least square estimator~\citep{AISTATS'19:window-LB} and the weighted least square estimator~\citep{NIPS'19:weighted-LB}, which will fix their results from $\Ot(d^{2/3}T^{2/3}P_T^{1/3})$ to $\Ot(d^{7/8}T^{3/4}P_T^{1/4})$. Also see related discussions in Remark~\ref{remark:estimate-error}.
 
\paragraph{Impossibility result.} In the following, we further prove that the desirable claim in~\eqref{eq:key-to-hold} is actually impossible. Specifically, we construct a hard problem instance to show that the key quantity $\sigma_{\max}\left(V_{t-1}^{-1} \big(\sum_{s=t_0}^{p} X_s X_s^\T\big)\right)$ cannot be universally upper bounded by any constant without square-root dependence on $H$. For notational convenience, we focus on the first restarting epoch, so the starting index $t_0 = 1$. 
\begin{myThm}
\label{thm:impossibility}
Let $L=1$ and $\lambda = 1$. We construct the feature as
\begin{equation}
  \label{eq:example}
  \begin{split}
  & X_1=\ldots=X_p = \left[\frac{1}{\sqrt{p}}, \frac{\sqrt{p-1}}{\sqrt{p}}\right]^{\T}, \mbox{ and } \\
  & X_{p+1}=\ldots=X_H = \left[\frac{1}{\sqrt{H-p}}, \frac{\sqrt{H-p-1}}{\sqrt{H-p}}\right]^{\T}.
  \end{split}
\end{equation}

Denote by $A = \sum_{s=1}^{p} X_s X_s^\T$ and $B = \sum_{s=p+1}^{H} X_s X_s^\T$, then the covariance matrix is $V_{t-1} = A + B +  I_d$. Under such cases, considering the checkpoint of $p = \floor{H/3}$, we have
\begin{equation}
  \label{eq:impossible}
  \norm{V_{t-1}^{-1} A}_2 = \sigma_{\max} (V_{t-1}^{-1} A) \geq 0.0564\sqrt{H}.
\end{equation}
\end{myThm}
\begin{proof}
For simplicity of notation, let $y = \sqrt{p-1}$ and $z = \sqrt{H-p-1}$. By the constructed example in~\eqref{eq:example}, we have
\begin{align*}
{} & A = \sum_{s=1}^{p} X_s X_s^\T = 
\begin{bmatrix}
  1 & y\\
  y & y^2
\end{bmatrix}
\mbox{ and } B = \sum_{s=p+1}^{H} X_s X_s^\T = 
\begin{bmatrix}
  1 & z\\
  z & z^2
\end{bmatrix}
.
\end{align*}

For convenience, we will write the covariance matrix $V_t$ simply $V$ when no confusion can arise. So the concerned matrix $V^{-1}A$ can be calculated as 
\begin{align*}
V^{-1}A = {} & 
\begin{bmatrix}
  2+\lambda & y+z\\
  y+z & y^2 + z^2 + \lambda
\end{bmatrix}^{-1}
\begin{bmatrix}
  1 & y\\
  y & y^2
\end{bmatrix}\\
= {} & \frac{1}{(2+\lambda)(y^2 + z^2 + \lambda) - (y+z)^2} \cdot  
\begin{bmatrix}
  y^2 + z^2 + \lambda & -(y+z)\\
  -(y+z) & 2+\lambda
\end{bmatrix}
\begin{bmatrix}
  1 & y\\
  y & y^2
\end{bmatrix}\\
= {} & \frac{1}{(1+\lambda)(y^2 + z^2) - 2yz + (2+\lambda)\lambda} \cdot  
\begin{bmatrix}
  z^2 - yz + \lambda & yz^2 - y^2z + \lambda y\\
  (1+\lambda)y-z & (1+\lambda)y^2-yz
\end{bmatrix}.
\end{align*}
Denote by $s= (1+\lambda)(y^2 + z^2) - 2yz + (2+\lambda)\lambda$, $\alpha = z^2 - yz + \lambda$, and $\beta = (1+\lambda)y-z$, we then have 
\begin{align*}
V^{-1}A (V^{-1}A)^\T = \frac{1+y^2}{s^2}
\begin{bmatrix}
  \alpha^2 & \alpha \beta\\
  \alpha \beta & \beta^2
\end{bmatrix}.
\end{align*}
The eigenvalues (we denote them by $\bar{\lambda}$, to distinguish the notation with the regularizer coefficient $\lambda$) of matrix $[\alpha^2, \alpha \beta; \alpha \beta, \beta^2]$ should satisfy $(\alpha^2 - \bar{\lambda})(\beta^2 - \bar{\lambda}) - \alpha^2 \beta^2 = 0$. By solving the equation, we can obtain that
\begin{align*}
   \bar{\lambda}_{\max} = \alpha^2 + \beta^2 = (z^2 - yz + \lambda)^2 + ((1+\lambda)y-z)^2 \geq (z^2 - yz + \lambda)^2.
\end{align*}
When $\lambda=1$ and $p = aH$ (here we assume $aH$ is an integer for simplicity), we have 
\begin{align}
  \bar{\lambda}_{\max} \geq {} & (z^2 - yz + \lambda)^2\nonumber \\
  = {} & \left( (1-a)p-1 - \sqrt{p-1}\sqrt{(1-a)p-1} + 1 \right)^2 \nonumber \\
  \geq {} & \left((1-a)H - \sqrt{a(1-a)}H\right)^2 \nonumber \\
  = {} & (1-a)(\sqrt{1-a} - \sqrt{a})^2 H^2 \label{eq:lb-1}.
\end{align}
Note that the condition of $a\in(0,1/2)$ is required to make the second inequality hold. On the other hand, we have
\begin{equation}
  \label{eq:lb-2}
  \frac{1+y^2}{s^2} = \frac{p}{(2(y^2 + z^2) - 2yz + 3)^2} \geq \frac{p}{(2(y^2 + z^2) +4)^2} = \frac{a}{4H}.
\end{equation}
Combining~\eqref{eq:lb-1} and~\eqref{eq:lb-2}, we have
\begin{align*}
  \sigma_{\max}(V^{-1}A) = \sqrt{\lambda_{\max} \left(V^{-1}A (V^{-1}A)^\T\right)} \geq \sqrt{\bar{\lambda}_{\max} \cdot \frac{1+y^2}{s^2}} \geq \sqrt{\frac{a'}{4}} \cdot \sqrt{H},
\end{align*}
where $a' = (1-a)a (\sqrt{1-a} - \sqrt{a})^2$ is a universal constant. When choosing $a = 1/3$ as selected in the main paper, $a' = 0.0127$ and the lower bound is $\sigma_{\max}(V^{-1}A) \geq 0.0564 \sqrt{H}$.
\end{proof}

Moreover, we report some numerical results for validation: when $H = 3000$, $\sigma_{\max} (V_{t-1}^{-1} A) = 5.852$ and the theoretical lower bound is $0.0564\sqrt{H} = 3.087$; when $H = 30000$, $\sigma_{\max} (V_{t-1}^{-1} A) = 18.474$ and the theoretical lower bound is $0.0564\sqrt{H} = 9.763$.

\subsection{Proof of Theorem~\ref{thm:dynamic-regret-BOB}}
\begin{proof}
We begin with the following decomposition of the  dynamic regret.
\begin{align*}
  \sum_{t=1}^T \langle X_t^*,\theta_t\rangle - \langle X_t,\theta_t\rangle = {} & \underbrace{\sum_{t=1}^T \langle X_t^*,\theta_t\rangle - \sum_{i=1}^{\ceil{T/\Delta}} \sum_{t = (i-1)\Delta + 1}^{i \Delta } \langle X_t(H^{\dagger}),\theta_t\rangle}_{\base}\\
  \quad {} & + \underbrace{\sum_{i=1}^{\ceil{T/\Delta}} \sum_{t = (i-1)\Delta + 1}^{i \Delta } \langle X_t(H^{\dagger}),\theta_t\rangle - \langle X_t(H_i),\theta_t\rangle}_{\meta},
\end{align*}
where $H^{\dagger}$ is the best restarting period to approximate the optimal restarting period $H^*$ in the pool $\H$, and $H^* = \floor{(dT/(P_T)^{2/3}}$. The first  term is the dynamic regret of RestartUCB with the best  restarting period in the candidate pool $\H$, and hence called base-regret. The second term is the regret overhead of meta-algorithm due to adaptive exploration of unknown optimal restarting period, and is thus called the meta-regret. We bound the two terms respectively.

We first consider the base-regret. Indeed, from the construction of candidate restarting periods pool $\H$, we confirm that there exists an restarting period $H^{\dagger} \in \H$ such that $H^{\dagger} \leq H^* \leq 2H^{\dagger}$. Therefore, employing the dynamic regret bound~\eqref{eq:main-result} in Theorem~\ref{thm:dynamic-regret}, we have the following upper bound for the base-regret:
\begin{align}
  \base \leq {} & \sum_{i=1}^{\ceil{T/\Delta}} \Ot\left(d^{\frac{1}{2}} H^{\dagger \frac{3}{2}} P_i + \frac{d\Delta}{\sqrt{H^\dagger}}\right) \label{eq:termb-step1} \\
  = {} & \Ot\left(d^{\frac{1}{2}} H^{\dagger \frac{3}{2}} P_T + \frac{dT}{\sqrt{H^\dagger}}\right) \label{eq:termb-step2} \\
  \leq {} & \Ot\left(d^{\frac{1}{2}} H^{* \frac{3}{2}} P_T + \frac{dT}{\sqrt{2H^*}}\right) \label{eq:termb-step3}\\
  = {} & \Ot \big( d^{\frac{7}{8}} T^{\frac{3}{4}} P_T^{\frac{1}{4}} \big), \label{eq:base-regret-bound}
\end{align}
where~\eqref{eq:termb-step1} is due to Theorem~\ref{thm:dynamic-regret} and $P_i$ denotes the path-length in the $i$-th episode of the meta-learner's update.~\eqref{eq:termb-step2} follows by summing over all update episodes, and the inequality~\eqref{eq:termb-step3} holds since the optimal restarting period $H^*$ is provably in the range of $[H_{\min}, H_{\max}]$ and satisfies $H^{\dagger} \leq H^* \leq 2H^{\dagger}$.

Next, we give an upper bound for the meta-regret. The analysis follows the proof argument in the sliding window based approach~\citep[Proposition~1]{arXiv'19:window-LB}. Note that the definition of the meta-regret is defined over the \emph{expected} reward (namely, $\E[r_t(X)] = X^\T \theta_t$), whereas the actual returned feedback is the noisy one (i.e., $r_t(X) = X^{\T} \theta_t + \eta_t$) which might be unbounded due to the additive sub-Gaussian noise. Fortunately, the light-tail property enables us to continue the use of adversarial MAB algorithms, e.g., Exp3~\citep{SICOMP'02:Auer-EXP3}. Specifically, by the concentration inequality endowed by the sub-Gaussian noise, we know that the received reward lies in the bounded region with high probability, which is presented in Lemma~\ref{lemma:bob}. Denote by $\mathcal{E}$ the event that Lemma~\ref{lemma:bob} holds, and denote by $R_i \triangleq \sum_{t = (i-1)\Delta + 1}^{i \Delta} \langle X_t(H^{\dagger}),\theta_t\rangle - \langle X_t(H_i),\theta_t\rangle$ the instantaneous regret of the meta learner. The meta-regret follows
\begin{align}
  \meta = {} & \E\left[\sum_{i=1}^{\ceil{T/\Delta}}  R_i\right] \nonumber \\
    = {} & \E\left[\sum_{i=1}^{\ceil{T/\Delta}} R_i~\Big\vert~\mathcal{E}\right] \cdot \Pr[\Ecal] + \E\left[\sum_{i=1}^{\ceil{T/\Delta}} R_i~\Big\vert~\overline{\Ecal}\right] \cdot \Pr[\overline{\Ecal}]\nonumber \\
    \leq {} & \O\sbr{L_{\max}\sqrt{\frac{T}{\Delta}N}} \cdot \sbr{1-\frac{2}{T}} + \O(T) \cdot \frac{2}{T}\nonumber \\
    = {} & \O(\Delta N T) \le \Ot(d^{1/2}T^{3/4}), \label{eq:meta-regret-bound}
\end{align}
where $L_{\max} \triangleq \max L_i$ for $i\in [\ceil{T/\Delta}]$ denotes the maximum cumulative loss in all episodes. The first equation is by definition, and the second one is by the law of total expectation. The next inequality follows from the following two aspects: the quantity is bounded according to the standard regret guarantee of Exp3~\citep{SICOMP'02:Auer-EXP3} when the event $\Ecal$ holds; and it is trivially upper bounded when the event $\Ecal$ does not happen. The failing probability is controlled by Lemma~\ref{lemma:bob}. The final equation is true by checking the parameters that the episode length is $\Delta = \ceil{d\sqrt{T}}$, and the number of candidate restarting periods $N$ is of order $\O(\log T)$, hence be omitted in the $\Ot(\cdot)$-notation.

Combining the upper bounds of base-regret~\eqref{eq:base-regret-bound} and meta-regret~\eqref{eq:meta-regret-bound}, we obtain that the expected dynamic regret of RestartUCB-BOB is bounded by $\Ot( d^{\frac{7}{8}} T^{\frac{3}{4}} P_T^{\frac{1}{4}})$, which completes the proof of Theorem~\ref{thm:dynamic-regret-BOB}. 

It is also worthy noting that the base-regret is actually with high-probability guarantees, while the meta-regret in our analysis only holds in expectation. Actually, we can boost the result to the high-probability version by employing advanced meta-algorithms such as Exp3.IX~\citep{NIPS'15:Neu} that can achieve a high-probability regret bound for adversarial MAB problems, as well as using the union bound in the analysis.
\end{proof}

\section{Experiments}
\label{sec:experiment}
Despite the focus of this paper is on the theoretical aspect, we present empirical studies to further evaluate the proposed approach.

\paragraph{Contenders.} We study two kinds of non-stationary environments: the underlying parameter is \emph{abruptly changing} or \emph{gradually changing}. We will simulate both environments and details can be found in the next paragraph. We compare RestartUCB to (a) WindowUCB, based on the sliding window least square~\citep{AISTATS'19:window-LB}; (b) WeightUCB, based on the weighted least square~\citep{NIPS'19:weighted-LB}; (c) StaticUCB, the algorithm designed for stationary linear bandits~\citep{NIPS'11:AY-linear-bandits}. In the scenario of abrupt change, we additionally compare with OracleRestartUCB, which knows the exact information of change points a priori and restarts the algorithm when reaching a change point. Evidently, OracleRestartUCB is not a practical algorithm, which actually serves as the skyline of all the approaches.

\paragraph{Settings.} In abruptly-changing environments, the unknown regression parameter $\theta_t$ is periodically set as $[1,0]$, $[-1,0]$, $[0,1]$, $[0,-1]$ in the first half of iterations, and $[1,0]$ for the remaining iterations. In gradually-changing environments, the unknown regression parameter  $\theta_t$ is moved from $[1,0]$ to $[-1,0]$ on the unit circle continuously. In both scenarios, we set $T=50,000$ and number of arms $n=20$. The feature is sampled from normal distribution $\mathcal{N}(0,1)$ and rescaled such that $L=1$. The random noise is generated according to $\mathcal{N}(0,0.1)$. Since the path-length $P_T$ is available in the synthetic datasets, we set the weight $\gamma = 1-1/\tau$ for WeightUCB, the window size $w = \tau$ for WindowUCB, and the restarting period $H = \tau$ for RestartUCB, here $\tau = 10*\floor{d^{1/4}T^{1/2} P_T^{-1/2}}$ is set as suggested by the theory. The simulation is repeated for $50$ times, and we report the average and standard deviation. 

\paragraph{Results.} Figure~\ref{figure:change} shows performance comparisons of different approaches for non-stationary linear bandits. The performance is measured by the (pseudo-) dynamic regret, which is plotted the in y-axis in the logarithmic scale. In the \emph{abruptly-changing environments}, OracleRestartUCB is definitely the best one as was expected since it knows exact information of change points a priori, and StaticUCB ranks the last as it does not take the non-stationarity issue into consideration. RestartUCB and WindowUCB have comparable performance, better than WeightUCB. Actually, RestartUCB is even slightly better than WindowUCB. We note that RestartUCB has an additional advantage over WindowUCB in terms of the computational issue: RestartUCB supports the one-pass update without storing historical data, whereas WindowUCB has to maintain a buffer and thus needs to scan data multiple times owing to the sliding window strategy. In the \emph{gradually-changing environments}, WeightUCB ranks the first, followed by WindowUCB and RestartUCB. Nevertheless, as will be shown later, WeightUCB takes a significantly longer running time than our approach.

Figure~\ref{figure:time} reports the running time including both mean and standard deviation. We can see that the time costs of RestartUCB, WindowUCB and StaticUCB are almost the same. By contrast, WeightUCB requires a significantly longer running time, nearly twice the cost of other contenders. The reason lies in the fact that WeightUCB algorithm involves the computation of inverse of covariance matrix $V_t \in \R^{d\times d}$ and its variant $\tilde{V}_t \in \R^{d\times d}$, while other three methods maintain and manipulate only one covariance matrix. It is worthy to note that our approach can be further accelerated by the recursive least square. This will save the inverse computation of the covariance matrix, which will be particularly desired in high-dimensional problems.

\begin{figure}[!t]
    \centering
    \subfloat[abrupt change]{ \label{figure:abrupt} 
        \includegraphics[clip, trim=3.5cm 9.2cm 4.3cm 9.9cm,height=0.35\textwidth]{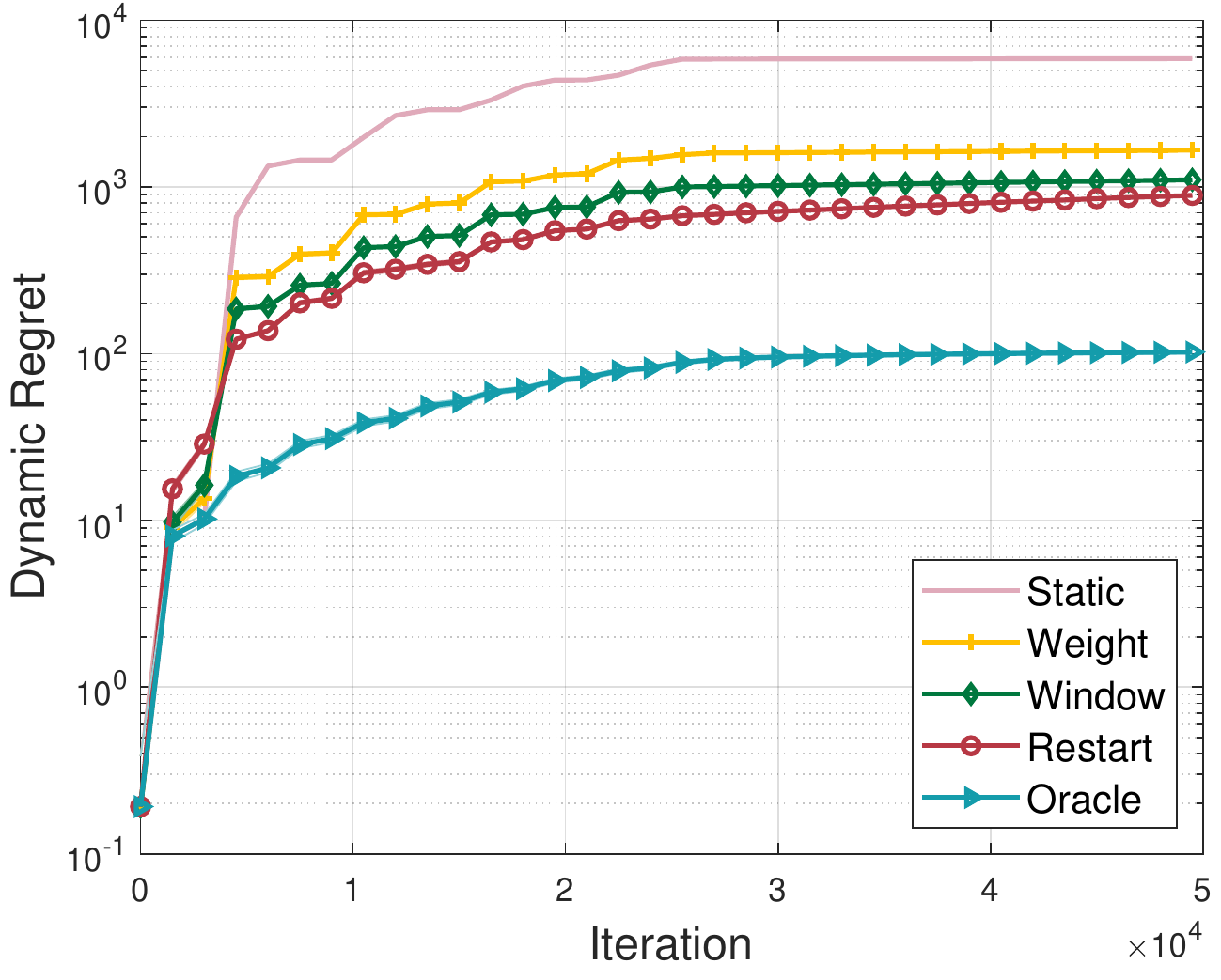}}\hspace{3mm}
    \subfloat[gradual change]{ \label{figure:slow}
        \includegraphics[clip, trim=3.5cm 9.2cm 2.4cm 9.9cm,height=0.35\textwidth]{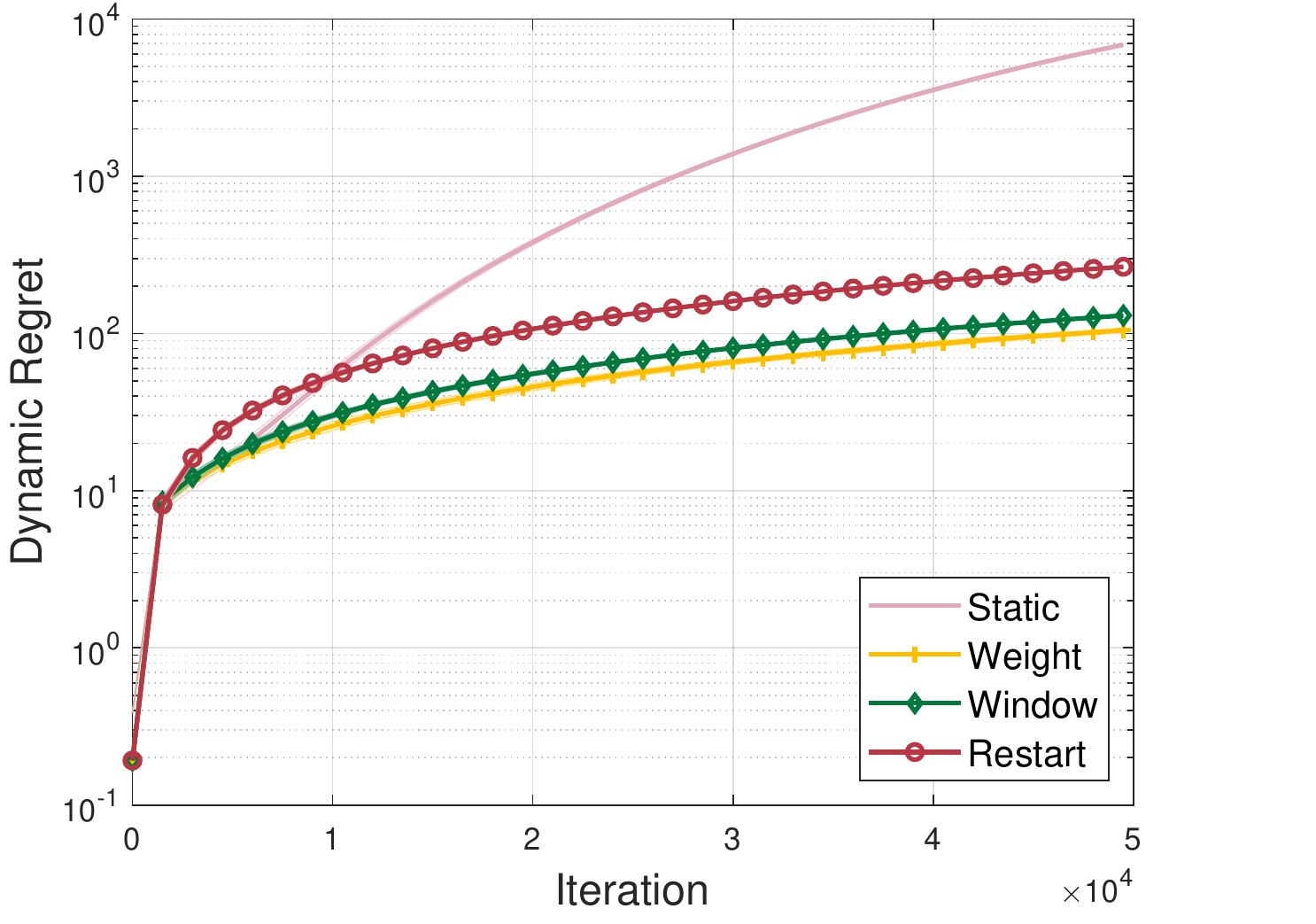}} 
    \caption{Comparisons of different approaches in terms of dynamic regret. Note that the y-axis is plotted in the logarithmic scale.}
    \label{figure:change}
\end{figure}

\begin{figure}[!t]
    \centering
    \includegraphics[clip, trim=3cm 2cm 4.5cm 2cm,width=0.65\textwidth]{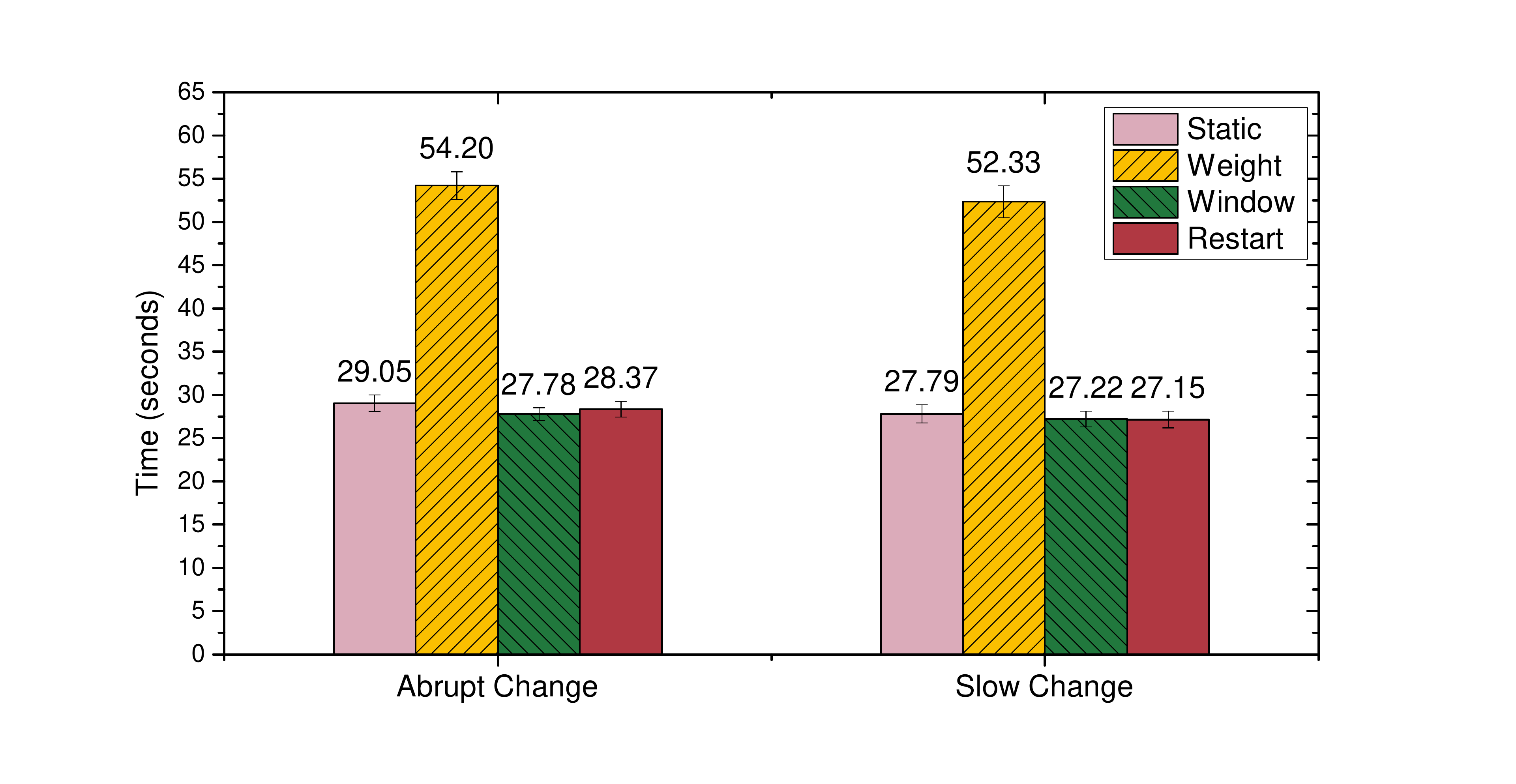}
    \caption{Comparisons of different approaches in terms of running time.}
    \label{figure:time}
\end{figure}

\section{Conclusion}
\label{sec:conclusion}
In this paper, we study the problem of non-stationary linear bandits, where the unknown regression parameter $\theta_t$ is changing over time. We propose a simple algorithm based on the restarted strategy, which enjoys strong theoretical guarantees notwithstanding its simplicity. Concretely, when the path-length of underlying parameters $P_T$ is known, our proposed RestartUCB algorithm enjoys an $\Ot(d^{7/8}T^{3/4}P_T^{1/4})$ dynamic regret, which shares the same regret guarantees with previous methods developed in the literature yet is with more favorable computational advantage. In addition, we show that the same dynamic regret guarantee is attainable even when $P_T$ is unknown by further using ResatartUCB as the base algorithm and combining the bandits-over-bandits mechanism as the meta scheduling. Empirical studies validate the efficacy of the proposed approach, particularly in the abruptly-changing environments.

The current upper bounds do not match the existing lower bound, even when the path-length term is known. In the future, we would like to investigate how to get rid of this regret gap and further study how to design algorithms for non-stationary linear bandits that achieve rate-optimal dynamic regret without prior information.

\section*{Acknowledgment}
This research was supported by the National Science Foundation of China (61921006, 61673201), and the Collaborative Innovation Center of Novel Software Technology and Industrialization. We are grateful for the anonymous reviewers for their helpful comments. The preliminary version of this paper appears at AISTATS 2020, in which we use the wrong argument~\eqref{eq:claim} as spotted in Section~\ref{sec:revisit}. This version has fixed the technical error. For the correction, the authors thank Jin-Hui Wu for many helpful discussions, especially on the impossibility result of Theorem~\ref{thm:impossibility}. We also acknowledge Yu-Hu Yan for carefully proofreading the paper.

\bibliography{online_learning}
\bibliographystyle{plainnat}

\appendix
\section{Technical Lemmas}
\label{appendix:tech-lemmas}
In this section, we provide several technical lemmas that frequently used in the proofs. 

\begin{myThm}[Self-Normalized Bound for Vector-Valued Martingales~{\citep[Theorem 1]{NIPS'11:AY-linear-bandits}}]
\label{thm:self-normalize}
Let $\{F_t\}_{t=0}^\infty$ be a filtration. Let $\{\eta_t\}_{t=0}^\infty$ be a real-valued stochastic process such that $\eta_t$ is $F_t$-measurable and conditionally $R$-sub-Gaussian for some $R>0$, namely,
\begin{equation}
  \label{eq:sub-Gaussian}
  \forall \lambda \in \mathbb{R}, \quad \mathbb{E}[\exp(\lambda \eta_t) \mid F_{t-1}] \leq \exp\left(\frac{\lambda^2 R^2}{2}\right).
\end{equation}

Let $\{X_t\}_{t=1}^\infty$ be an $\mathbb{R}^d$-valued stochastic process such that $X_t$ is $F_{t-1}$-measurable. Assume that $V$ is a $d\times d$ positive definite matrix. For any $t\geq 0$, define
\begin{equation}
  \label{eq:covariance-matrix}
  \bar{V}_t = V + \sum_{\tau = 1}^t X_\tau X_\tau^{\T},\quad S_t = \sum_{\tau = 1}^t \eta_\tau X_\tau.
\end{equation}

Then, for any $\delta>0$, with probability at least $1-\delta$, for all $t\geq 0$, 
\begin{equation}
  \label{eq:self-normal-concentration}
  \norm{S_t}_{\bar{V_t}^{-1}}^2 \leq 2R^2 \log \left(\frac{\det(\bar{V_t})^{1/2} \det(V)^{-1/2}}{\delta}\right).
\end{equation}
\end{myThm}

\begin{myLemma}[Elliptical Potential Lemma]
    \label{lemma:potential}
    Suppose $U_0 = \lambda I$, $U_t = U_{t-1} + X_tX_t^{\T}$, and $\norm{X_t}_2 \leq L$, then
    \begin{equation}
      \label{eq:potential}
      \sum_{t=1}^T \lVert U_{t-1}^{-\frac{1}{2}} X_t\rVert_2 \leq \sqrt{2dT \log\left(1 + \frac{L^2T}{\lambda d}\right)}.
    \end{equation}
\end{myLemma}

\begin{proof}
    First, we have the following decomposition,
    \[
        U_t = U_{t-1} + X_tX_t^{\T} = U_{t-1}^{\frac{1}{2}}(I + U_{t-1}^{-\frac{1}{2}}X_t X_t^{\T}U_{t-1}^{-\frac{1}{2}})U_{t-1}^{\frac{1}{2}}.
    \]

    Taking the determinant on both sides, we get
    \[
    \det(U_t) = \det(U_{t-1}) \det(I + U_{t-1}^{-\frac{1}{2}}X_t X_t^{\T}U_{t-1}^{-\frac{1}{2}}),
    \]
    which in conjunction with Lemma~\ref{lemma:determinant} yields
    \[
       \det(U_t) = \det(U_{t-1}) (1 + \lVert U_{t-1}^{-\frac{1}{2}} X_t\rVert_2^2) \geq \det(U_{t-1}) \exp(\lVert U_{t-1}^{-\frac{1}{2}} X_t\rVert_2^2/2).
  \]
    Note that in the first inequality, we utilize the fact that $1 + x \geq \exp(x/2)$ holds for any $x\in [0,1]$. By taking advantage of the telescope structure, we have
    \[
      \begin{split}
             & \sum_{t=1}^T \lVert U_{t-1}^{-\frac{1}{2}} X_t\rVert_2^2 \leq 2\log \frac{\det(U_T)}{\det(U_0)} \leq 2d\log \left(1 + \frac{L^2T}{\lambda d}\right),
      \end{split}
    \]
    where the last inequality follows from the fact that $\mbox{Tr}(U_T) \leq \mbox{Tr}(U_0) + L^2T = \lambda d + L^2T$, and thus $\det(U_T) \leq (\lambda + L^2T/d)^d$. Therefore, Cauchy-Schwarz inequality implies,
    \[
      \sum_{t=1}^T \lVert U_{t-1}^{-\frac{1}{2}} X_t\rVert_2 \leq \sqrt{T\sum_{t=1}^T \lVert U_{t-1}^{-\frac{1}{2}} X_t\rVert_2^2} \leq \sqrt{2dT\log\left(1 + \frac{L^2T}{\lambda d}\right)}.
    \]
\end{proof}

\begin{myLemma}
    \label{lemma:determinant}
    For any $\v \in \R^d$, we have
    \[
    \det(I+\mathbf{v} \mathbf{v}^{\T}) =  1 + \norm{\mathbf{v}}_2^2.
    \]
\end{myLemma}

\begin{proof}
Notice that 
\begin{enumerate}
\item[(i)] $(I+\mathbf{v} \mathbf{v}^\T)\mathbf{v} = (1 + \norm{\mathbf{v}}_2^2)\mathbf{v}$, therefore, $\mathbf{v}$ is its eigenvector with $(1 + \norm{\mathbf{v}}_2^2)$ as the eigenvalue; 
\item[(ii)] $(I + \mathbf{v} \mathbf{v}^\T)\mathbf{v}^{\perp} =  \mathbf{v}^{\perp}$, therefore, $\mathbf{v}^{\perp} \perp \mathbf{v}$ is its eigenvector with $1$ as the eigenvalue. 
\end{enumerate}
Consequently, $\det( I+\mathbf{v} \mathbf{v}^\T) = 1 + \norm{\mathbf{v}}_2^2$.
\end{proof}

\begin{myLemma}[{Property 5.2.9 of~\citet{meyer2000matrix}}]
\label{lemma:matrix-2-norm}
For a real matrix $A \in \R^{m \times n}$, we have
\[
  \norm{A}_2 = \sup_{\norm{\x}_2 = 1} \sup_{\norm{\y}_2 = 1} \abs{\y^{\T} A \x}.
\] 
\end{myLemma}
\begin{proof}
The proof is from the solution manual of~\citet{meyer2000matrix}. Applying the Cauchy-Schwarz inequality yields $\abs{\y^\T A \x} \leq \norm{\y}_2 \norm{A\x}_2$, which implies that
\[
  \sup_{\norm{\x}_2 = 1} \sup_{\norm{\y}_2 = 1} \abs{\y^{\T} A \x} \leq \sup_{\norm{\x}_2 = 1} \norm{A \x}_2 = \norm{A}_2.
\]
Now show that equality is actually attained for some pair $\x$ and $\y$ on the unit $2$-sphere. To do so, notice that when setting $\x_*$ is a vector of unit length such that
\[
  \norm{A\x_*}_2 = \sup_{\norm{\x}=1} \norm{A\x}_2 = \norm{A}_2,
\]
and $\y_*$ is the vector such that
\[
  \y_* = \frac{A \x_*}{\norm{A \x_*}_2} = \frac{A \x_*}{\norm{A}_2},
\]
then
\[
  \y_*^\T A \x_* = \frac{\x_*^\T A^\T A \x_*}{\norm{A}_2} = \frac{\norm{A \x_*}_2^2}{\norm{A}_2} = \frac{\norm{A}_2^2}{\norm{A}_2} = \norm{A}_2.
\]
Hence we complete the proof.
\end{proof}

\begin{myLemma}
  \label{lemma:bob}
  Let $N= \ceil{T/\Delta}$. Denote by $L_i$ the absolute value of cumulative rewards for episode $i$, i.e., $L_i \triangleq \sum_{t = (i-1)\Delta + 1}^{i \Delta} r_t(X_t)$, then 
  \begin{equation}
  \label{eq:concentration}
    \Pr\left[\forall i\in [N], L_i\leq LS\Delta+2R\sqrt{\Delta\ln\frac{T}{\sqrt{\Delta}}}\right] \geq 1-\frac{2}{T}.
  \end{equation}
\end{myLemma}

\begin{proof}
  For any episode $i$, the absolute sum of rewards can be written as
\begin{align*}
  \left|\sum_{t=(i-1)\Delta+1}^{i\Delta}\langle X_t,\theta_t\rangle + \eta_t\right| \leq {}& \sum_{t=(i-1)\Delta+1}^{i\Delta} \left|\langle X_t,\theta_t\rangle\right|+\left|\sum_{t=(i-1)\Delta+1}^{i\Delta}\eta_t\right|\\
  \leq {}& \Delta LS+\left|\sum_{t=(i-1)\Delta+1}^{i\Delta}\eta_t\right|,
\end{align*}
where we have iteratively applied the triangle inequality as well as the fact that $\left|\langle X_t,\theta_t\rangle\right|\leq LS$ for all $t$. 

Further applying the standard concentration result of  $R$-sub-Gaussian random variables~\citep[Corollary 1.7]{2019:HighDimension-book}, we get
Now by property of the $R$-sub-Gaussian, it holds that
\[
  \Pr\left[\left|\frac{1}{\Delta}\sum_{t=(i-1)\Delta+1}^{i\Delta}\eta_t\right| \geq \epsilon\right] \leq 2\exp\left( -\frac{\Delta \epsilon^2}{2R^2} \right),
\]
which further implies that
\[
  \Pr\left[\left|\sum_{t=(i-1)\Delta+1}^{i\Delta}\eta_t\right|\geq 2R\sqrt{\Delta\ln\frac{T}{\sqrt{\Delta}}}\right]\leq\frac{2\Delta}{T^2}.
\]
So we can ensure a low failing probability, specifically, the probability of the event that the absolute value of the noise term $\eta_t$ exceeds $2R\sqrt{\ln T}$ for a fixed $t$ is at most $1/T^2$. By union bound, we have
\begin{align*}
  & \Pr\left[\exists i\in [N]:\left|\sum_{t=(i-1)\Delta+1}^{i\Delta}\eta_t\right|\geq 2R\sqrt{\Delta\ln\frac{T}{\sqrt{\Delta}}}\right]\\
  \leq {} & \sum_{i=1}^{\lceil T/\Delta\rceil}\Pr\left[\left|\sum_{t=(i-1)\Delta+1}^{i\Delta}\eta_t\right|\geq 2R\sqrt{\Delta\ln\frac{T}{\sqrt{\Delta}}}\right]\leq\frac{2}{T}.
\end{align*}
Hence, we finish the proof.
\end{proof}
\end{document}